%% file: main-arxiv.tex
\documentclass{article}
\usepackage[utf8]{inputenc} 
\usepackage[T1]{fontenc}    
\usepackage{fullpage}

\usepackage[round]{natbib}

\usepackage[margin = 1in]{geometry}

\input{format/_packages-arxiv.tex}
\input{format/_shortcuts.tex}
\input{format/_notations.tex}

\title{Orthogonal Statistical Learning with Self-Concordant Loss}
\author{Lang Liu \qquad Carlos Cinelli \qquad Zaid Harchaoui \vspace{0.3cm} \\
Department of Statistics, University of Washington}
\date{}

\begin{document}

\maketitle

\begin{abstract}
  Orthogonal statistical learning and double machine learning have emerged as general frameworks for two-stage statistical prediction in the presence of a nuisance component. We establish non-asymptotic bounds on the excess risk of orthogonal statistical learning methods with a loss function satisfying a self-concordance property. Our bounds improve upon existing bounds by a dimension factor while lifting the assumption of strong convexity. We illustrate the results with examples from multiple treatment effect estimation and generalized partially linear modeling.
\end{abstract}

\section{Introduction}\label{sec:intro}
\input{sections/sec1}

\section{Main Results}\label{sec:thms}
\input{sections/sec2}

\section{Discussion}\label{sec:discussion}
\input{sections/sec3}

\section{Applications and Examples}\label{sec:application}
\input{sections/sec4}

\section{Conclusion}\label{sec:conclusion}
\input{sections/sec5}

\clearpage

\subsubsection*{Acknowledgements}
The authors would like to thank Jon Wellner for fruitful discussions.
L. Liu is supported by NSF CCF-2019844 and NSF DMS-2023166.
Z. Harchaoui is supported by NSF CCF-2019844, NSF DMS-2134012, NSF DMS-2023166, CIFAR-LMB, and faculty research awards.
Part of this work was done while Z. Harchaoui was visiting the Simons Institute for the Theory of Computing.

\bibliographystyle{abbrvnat}
\bibliography{biblio}

\clearpage
\appendix

The Appendix is organized as follows.
For simplicity, we first prove in \Cref{sec:known_nuisance} the main results assuming that the true nuisance parameter $g_0$ is known.
We then prove in \Cref{sec:proof} the main results presented in \Cref{sec:thms}.
The technical tools used in the proofs are reviewed and developed in \Cref{sec:tools}.

\section{Risk Bound with Known Nuisance Parameter}
\label{sec:known_nuisance}
\input{sections/known_nuisance}

\section{Proof of Theorem \ref{thm:risk_bound}}\label{sec:proof}
\input{sections/proof}

\section{Technical Tools}\label{sec:tools}
\input{sections/tools}

\end{document}

%% file: format/_packages-arxiv.tex
\usepackage{amsmath, amssymb, amsfonts}
\usepackage{amsthm, thmtools, thm-restate}
\usepackage{mathrsfs}
\usepackage{mathtools, bm, enumitem}
\usepackage{algorithm, algorithmicx, algpseudocode}
\usepackage{times}
\usepackage{color}
\usepackage[table]{xcolor}  %
\usepackage{dsfont}
\usepackage{booktabs}

\usepackage{minitoc} 

\newtheorem{theorem}{Theorem}
\newtheorem{proposition}[theorem]{Proposition}
\newtheorem{lemma}[theorem]{Lemma}

\newtheorem{definition}{Definition}
\newtheorem{assumption}{Assumption}

\newtheorem{remark}{Remark}

\definecolor{puorange}{rgb}{0.80,0.20,0}
\definecolor{bluegray}{rgb}{0.04,0,0.7}
\definecolor{greengray}{rgb}{0.05,0.50,0.15}
\definecolor{darkbrown}{rgb}{0.40,0.2,0.05}
\definecolor{darkcyan}{rgb}{0,0.4,1}
\definecolor{black}{rgb}{0,0,0}
\definecolor{grey}{rgb}{0.93,0.93,0.93}
\definecolor{royalazure}{rgb}{0.0, 0.22, 0.66}

\usepackage[colorlinks,citecolor=bluegray,linkcolor=darkbrown,urlcolor=blue,breaklinks]{hyperref}
\usepackage[capitalise]{cleveref}
\Crefname{assumption}{Assumption}{Assumption}
\Crefname{proposition}{Proposition}{Proposition}

\newtheorem{innercustomasmp}{Assumption}
\newenvironment{customasmp}[1]
  {\renewcommand\theinnercustomasmp{#1}\innercustomasmp}
  {\endinnercustomasmp}

\usepackage{tikz}
\usetikzlibrary{bayesnet}
\usetikzlibrary{shapes,decorations,arrows,calc,arrows.meta,fit,positioning}
\tikzset{
	-latex,auto,node distance = 1.3cm and 1.3cm,
	semithick ,
	state/.style ={ellipse, draw, minimum width = 0.5 cm},
	el/.style = {inner sep=2pt, align=left, sloped},
	point/.style = {circle, draw=black, inner sep=0.04cm,
     fill=black,node contents={}},
    bidirected/.style={Latex-Latex,dashed},
      square/.style={rectangle, 
                 draw = black, 
                 inner sep = 0.08cm,
                 node distance = 0.8cm and 0.8cm,
                 fill=black,  
                 node contents={}}
}

%% file: format/_shortcuts.tex
\newcommand{\reals}{{\mathbb R}}

\newcommand{\abs}[1]{\left| #1 \right|}
\newcommand{\norm}[1]{\left\lVert #1 \right\rVert}
\newcommand{\anorm}[1]{\left| #1 \right|}

\newcommand{\Rank}{\operatorname{\bf Rank}}
\newcommand{\Tr}{\operatorname{\bf Tr}}

\newcommand{\ip}[1]{{\langle #1 \rangle}}

\newcommand \D {\mathrm{d}}

\newcommand{\Expect}{\operatorname{\mathbb E}}
\newcommand{\Var}{\operatorname{\mathbb{V}ar}}

\newcommand{\Cov}{\operatorname{\mathbb{C}ov}}
\newcommand{\Prob}{\operatorname{\mathbb P}}

\newcommand{\indep}{\perp \!\!\! \perp}

\newcommand{\calA}{\mathcal{A}}
\newcommand{\calB}{\mathcal{B}}
\newcommand{\calC}{\mathcal{C}}

\newcommand{\calF}{\mathcal{F}}
\newcommand{\calG}{\mathcal{G}}

\newcommand{\calK}{\mathcal{K}}

\newcommand{\calN}{\mathcal{N}}

\newcommand{\calX}{\mathcal{X}}

\newcommand{\calZ}{\mathcal{Z}}

\newcommand{\bbN}{{\mathbb N}}

\newcommand{\ind}{\mathds{1}}

\newcommand{\txtover}[2]{\overset{\mbox{\scriptsize #1}}{#2}}

%% file: format/_notations.tex
\newcommand{\erisk}{\mathcal{E}}
\newcommand{\sample}{\mathcal{D}}
\newcommand{\rmD}{\mathrm{D}}

%% file: sections/sec1.tex
As statistical machine learning impacts several domain applications of major importance to the planet and society, ranging from healthcare to the environment, sophisticated approaches to estimation, proceeding in multiple stages, are being developed to overcome confounding factors and to address high-dimensional nuisance parameters~\citep{peters2017elements}. Orthogonal statistical learning (OSL), and its statistical estimation predecessor double machine learning (DML), have emerged as general frameworks for two-stage statistical machine learning in the presence of a nuisance component~\citep{mackey2018orthogonal,liu2021double,nekipelov2022regularised}. 

The power of this framework can be illustrated on the task of 
assessing the causal effect of a treatment on an outcome of interest. Let $Z := (Y, D, X)$ be a vector of observed variables, where $Y \in \reals$ is the outcome, $D \in \{0, 1\}$ is the treatment, and $X \in \reals^p$ is a vector of features. Denote by $Y(\mathrm{d})$ the potential outcome of $Y$ when the treatment variable $D$ is set (by intervention) to be $\mathrm{d} \in \{0, 1\}$. Our goal is to estimate the average treatment effect (ATE) of $D$ on $Y$, defined as $ \theta_0 := \Expect[Y(1) - Y(0)]$.

If the treatment assignment $D$ is conditionally ignorable (unconfounded) given $X$; or, equivalently, if the  set of features $X$ satisfy the ``backdoor'' (adjustment) criterion for estimating the causal effect of $D$ on $Y$ (see Figure~\ref{fig:unconfoundedness} for an illustrative causal diagram), a well known identification result in the causal inference literature is that the ATE $\theta_0$ can be identified as a functional of the conditional expectation function (CEF) of the outcome \citep{Rosenbaum1983,pearl2009causality,shpitser:apa2012,imbens2015causal,hernan2010causal}.
To be more concrete, we obtain that $\theta_0 = \Expect\big[\Expect[Y \mid D = 1, X] - \Expect[Y \mid D = 0, X] \big]$.
Note that, in order to estimate $\theta_0$, which is a scalar, we may need to learn the potentially infinite dimensional nuisance $g := (g_0, g_1)$ where $g_d := \Expect[Y \mid D = \mathrm{d}, X]$.
This type of challenges, where in order to learn about the target of inference, one needs to estimate many quantities that are not of primary interest, is the one OSL and DML both seek to address. 

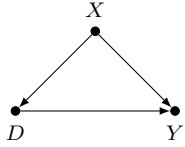
\begin{figure*}[t]
\centering
    \begin{tikzpicture}
        \node (z) [label = above:$X$, point];
        \node (x) [label = below:$D$, below left of = z, point, node distance = 1.5cm];
        \node (y) [label = below:$Y$, point, below right of = z, node distance = 1.5cm];
        \path (z) edge  (x);
        \path (x) edge  (y);
        \path (z) edge  (y);
    \end{tikzpicture}%
    \caption{Causal diagram in which $X$ satisfies the backdoor criterion for the causal effect of $D$~on~$Y$. Here, unconfoundedness holds conditional on $X$, that is, $Y(\mathrm{d}) \indep D \mid X$.}
    \vspace{-0.2in}
\label{fig:unconfoundedness}
\end{figure*}

We work in the framework of OSL and state our results in terms of excess risk in the spirit of statistical learning theory.
Formally, let $\sample := \{Z_1, \dots, Z_{2n}\}$ be an i.i.d.~sample of size $2n$ from an unknown distribution $\Prob$ on $\mathcal{Z}$.
We are interested in learning parameters from the model $\mathcal{M}_{\theta, g}$ equipped with some loss function $\ell(\theta, g; z)$, where $\theta \in \Theta \subset \mathbb{R}^d$ is the \emph{target parameter} and $g \in (\mathcal{G}, \norm{\cdot}_\calG)$ is the \emph{nuisance parameter} which may be infinite dimensional.
Define the \emph{population risk} at $(\theta, g)$ as
$L(\theta, g) := \Expect_{Z \sim \Prob}[\ell(\theta, g; Z)]$.
We will assume throughout that $\ell$ is three times differentiable w.r.t.~$\theta$ and twice differentiable w.r.t.~$g$.

Following \cite{foster2020orthogonal}, we assume that there exists a true nuisance parameter $g_0 \in \calG$.
Without access to $g_0$, we aim to learn an estimator $\hat \theta$ that minimizes the \emph{excess risk},
\begin{align}\label{eq:excess_risk}
    \erisk(\theta, g_0) := L(\theta, g_0) - \inf_{\theta \in \Theta} L(\theta, g_0).
\end{align}
We assume that the infimum in the excess risk is attainable at a minimizer $\theta_\star$ and the Hessian of $L(\cdot, g_0)$ at $\theta_\star$ is invertible. Consequently, we can rewrite \eqref{eq:excess_risk} as
\begin{align*}
    \erisk(\theta, g_0) = L(\theta, g_0) - L(\theta_\star, g_0).
\end{align*}

We focus on the two-stage learning procedure with sample splitting from~\citet{foster2020orthogonal}; see also \cite{chernozhukov2018double}.
Denote by $\sample_1 := \{Z_i\}_{i=1}^n$ the first sample split, and by $\sample_2 := \{Z_i\}_{i=n+1}^{2n}$ the second sample split. The estimator we study in this paper is constructed from the following algorithm.

\vspace{0.1in}
\noindent\fbox{
  \parbox{0.965\textwidth}{
    \textbf{OSL Meta-Algorithm}
    \begin{itemize}
        \item \emph{Nuisance parameter.} The first stage learning algorithm takes $\sample_2$ as input and outputs an estimator $\hat g$.
        \item \emph{Target parameter.} The second stage learning algorithm solves the minimization problem
        \begin{align}
            \min_{\theta \in \Theta} L_n(\theta, \hat g) := \frac1n \sum_{i=1}^n \ell(\theta, \hat g; Z_i)
        \end{align}
        and outputs an estimator $\hat \theta$.
    \end{itemize}
  }
}
\vspace{0.1in}

The main contribution of this paper is establishing non-asymptotic guarantees on the excess risk $\erisk(\hat \theta, g_0)$ for the OSL estimator $\hat{\theta}$ under a uniform self-concordance assumption, allowing the dimension of the target parameter to grow at the rate $d = O(n^{1/2})$. In particular, Theorems~\ref{thm:risk_bound}~and~\ref{thm:risk_bound_slow} derive novel non-asymptotic bounds for the excess risk and characterize its convergence as $n \rightarrow \infty$, both in a ``fast'' and ``slow'' regime. Compared to previous work, such as \citep{foster2020orthogonal}, these new bounds depend on the ``effective dimension'' as defined by the trace of the sandwich covariance matrix, and recover guarantees that were only available to supervised learning without a nuisance parameter. Effectively, this improves  prior  bounds on the excess risk  at least by a factor of $d$ in a wide range of eigendecay regimes.

In what follows, Section~\ref{sec:thms} provides the main definitions, assumptions, and establishes the main results of this paper. Section~\ref{sec:discussion} provides further discussions on the converge rate and how our work relates to existing literature.
\Cref{sec:application} examines concrete examples such as treatment effect estimation in a partially linear model and semi-parametric logistic regression.
Finally, Section~\ref{sec:conclusion} offers some concluding remarks. The full proofs are collected in the Appendix sections.

%% file: sections/sec2.tex
We first introduce the notation and some key definitions.
We then present all the assumptions required by our analysis.
Finally, we summarize our main results and their proof sketches.

\subsection{Preliminaries}

\paragraph{Notation.}
Let $S(\theta, g; z) := \nabla_\theta \ell(\theta, g; z)$ be the gradient at $z$ and $H(\theta, g; z) := \nabla_\theta^2 \ell(\theta, g, z)$ be the Hessian at $z$.
We also call $S(\theta, g; z)$ the score at $z$ which is named after the \emph{likelihood score} in maximum likelihood estimation.
Their population counterparts are $S(\theta, g) := \Expect_{Z \sim \Prob}[S(\theta, g; Z)]$ and $H(\theta, g) := \Expect_{Z \sim \Prob}[H(\theta, g; Z)]$.
We assume standard regularity assumptions so that $S(\theta, g) = \nabla_\theta L(\theta, g)$ and $H(\theta, g) = \nabla_\theta^2 L(\theta, g)$.
Moreover, we let $G(\theta, g) := \Cov_{Z \sim \Prob}(S(\theta, g; Z))$ be the covariance matrix of the score $S(\theta, g; Z)$.
For simplicity of the notation, we let $S_\star := S(\theta_\star, g_0)$, $G_\star := G(\theta_\star, g_0)$, and $H_\star := H(\theta_\star, g_0)$.
We define their empirical quantities as $S_n(\theta, g) := \frac1n \sum_{i=1}^n S(\theta, g; Z_i)$, $H_n(\theta, g) := \frac1n \sum_{i=1}^n H(\theta, g; Z_i)$, and
\begin{align*}
    G_n(\theta, g) := \frac1n \sum_{i=1}^n [S(\theta, g; Z_i) - S(\theta, g)] [S(\theta, g; Z_i) - S(\theta, g)]^\top.
\end{align*}
Our analysis is local to a \emph{Dikin ellipsoid} at $\theta_\star$ of radius $\bar r_1 := r_1 \sqrt{\lambda_{\min}(H_*)}$ and a ball at $g_0$ of radius $r_2$, i.e.,
\begin{align*}
    \Theta_{\bar r_1}(\theta_\star) := \{\theta \in \Theta: \norm{\theta - \theta_\star}_{H_\star} < \bar r_1\} \quad \mbox{and} \quad \calG_{r_2}(g_0) := \{g \in \calG: \norm{g - g_0}_{\calG} < r_2\},
\end{align*}
where, given a positive semi-definite matrix $J$, we let $\norm{x}_J := \norm{J^{1/2} x}_2 = \sqrt{x^\top J x}$.

\paragraph{Effective dimension.}
The quantity that plays a central role in our analysis is the \emph{profile effective dimension} defined as follows.
The term \emph{profile} is used in the same sense as in the profile likelihood literature; see, e.g., \cite{murphy2000profile}.
\begin{definition}
    We define the profile effective dimension to be
    \begin{align}
        \bar d_\star := \sup_{g \in \calG_{r_2}(g_0)} \Tr(H_\star^{-1/2} G(\theta_\star, g) H_\star^{-1/2}).
    \end{align}
\end{definition}
When the model is well-specified, we have $H_\star = G_\star$ and thus $\bar d_* \approx d$.
When the model is mis-specified, it corresponds to the mismatch between the covariance matrix $G_*$ and the Hessian matrix $H_*$.
It can be either as small as a constant or as large as exponential in $d$ depending on the eigendecays of $G_*$ and $H_*$; see \Cref{sec:discussion} for more details.

\paragraph{Self-concordance.}
We shall use the notion of \emph{self-concordance} from convex optimization.
Self-concordance was introduced to analyze the interior-point and Newton-type convex optimization algorithms~\citep{yurii1994interior}.
\citet{bach2010self} introduced a modified version, which we call the \emph{pseudo self-concordance}, to derive non-asymptotic bounds for the logistic regression.
We focus here on the pseudo self-concordance.
For a functional $F$ mapping from a vector space $\calF$ to $\reals$, we define the \emph{derivative operator} $\rmD$ as $\rmD F(f)[h] := \frac{\D }{\D t} F(f + th) \vert_{t = 0}$ for $f, h \in \calF$.

\begin{definition}\label{def:modify_self_concord}
    Let $\calX \subset \reals^d$ be open and $f: \calX \rightarrow \reals$ be a closed convex function.
    We say $f$ is pseudo self-concordant with parameter $R$ on $\calX$ if
    \begin{align*}
        \abs{\rmD^3 f(x)[u, u, u]} \le R \norm{u}_2 \rmD^2 f(x)[u, u], \quad \mbox{for all } x \in \calX, u \in \reals^d.
    \end{align*}
\end{definition}

\begin{figure}[t]
    \centering
    \includegraphics[width=0.4\textwidth]{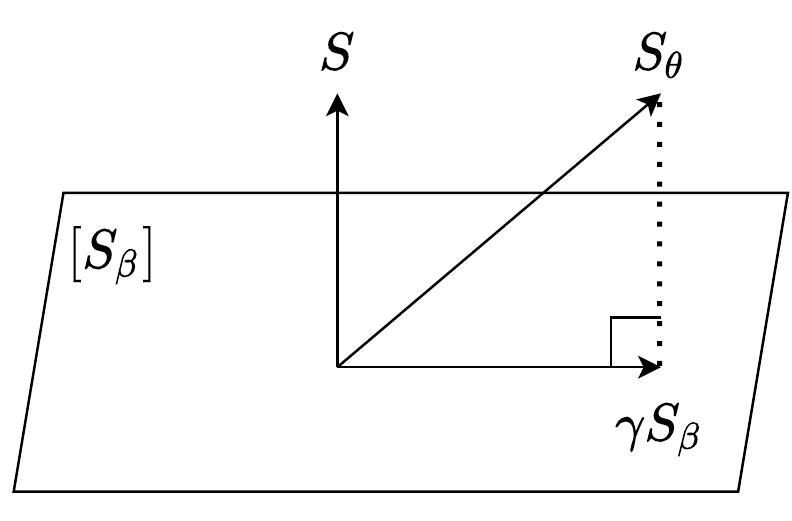}
    \caption{Illustration of the orthogonal score by projection.}
    \label{fig:ortho-score}
\end{figure}

\paragraph{Neyman orthogonality.}
We use \emph{Neyman orthogonality} \citep{neyman1959optimal,neyman1979test} to obtain a fast rate for the excess risk.
The intuition behind it is that we want the risk to be insensitive to perturbations in the nuisance $g$ so that a good estimate $\hat \theta$ can be obtained even if $\hat g$ is of poor quality.
\begin{definition}
    We say the population risk $L$ is Neyman orthogonal at $(\theta_\star, g_0)$ over $\Theta' \times \calG'$ if
    \begin{align}\label{eq:orthogonal}
        \rmD_g \rmD_\theta L(\theta_\star, g_0)[\theta - \theta_\star, g - g_0] = 0, \quad \mbox{for all } \theta \in \Theta', g \in \calG'.
    \end{align}
    Since \eqref{eq:orthogonal} also implies that $\rmD_g S(\theta_\star, g_0)[g - g_0] = 0$ for all $g \in \calG'$,
    we will also say the score $S$ is Neyman orthogonal at $(\theta_\star, g_0)$.
\end{definition}
When $g$ is parametrized by a finite-dimensional vector $\beta$, we can obtain a Neyman orthogonal score by projection.
Let $L(\theta, \beta)$ be some population risk which may not be Neyman orthogonal.
We project $S_\theta := \nabla_\theta L(\theta, \beta)$ onto the space spanned by $S_\beta := \nabla_\beta L(\theta, \beta)$ and obtain $S := S_\theta - \gamma S_\beta$ where $\gamma := [\nabla_{\theta} \nabla_\beta L(\theta, \beta)] [\nabla^2_{\beta} L(\theta, \beta)^{-1}]$.
It can be shown that $S$ is Neyman orthogonal at $(\theta_\star, \beta_0)$.
This procedure is illustrated in \Cref{fig:ortho-score}.
Now, to get a population risk that satisfies Neyman orthogonality, it suffices to take the integral of $S$ w.r.t.~$\theta$.

\subsection{Assumptions}
\label{sub:asmp}
Since our analysis is local to neighborhoods of $\theta_\star$ and $g_0$,
our first assumption localizes the estimator $\hat \theta$ and $\hat g$ to such neighborhoods.

\begin{assumption}[Localization]\label{asmp:local}
    Let $r_1, r_2 > 0$ be constants and $\bar r_1 := r_1 \sqrt{\lambda_{\min}(H_\star)} > 0$.
    There exists a function $N_{\bar r_1, r_2}: [0, 1] \rightarrow \bbN_+$ and such that for any $\delta \in (0, 1)$ we have, with probability at least $1 - \delta$, $\hat \theta \in \Theta_{\bar r_1}(\theta_\star)$ and $\hat g \in \calG_{r_2}(g_0)$ for all $n \ge N_{\bar r_1, r_2}(\delta)$.
\end{assumption}

The localization assumption is necessary to avoid a global strong convexity assumption which is assumed by \cite{foster2020orthogonal}.
In order to control the empirical score, we assume that the normalized score at $\theta_\star$ is sub-Gaussian uniformly over $\calG_{r_2}(g_0)$.
\begin{assumption}[Score sub-Gaussianity]\label{asmp:score_subg}
    There exists a constant $K_1 > 0$ such that,
    for every $g \in \calG_{r_2}(g_0)$,
    we have $\norm{G(\theta_\star, g)^{-1/2} [S(\theta_\star, g; Z) - S(\theta_\star, g)]}_{\psi_2} \le K_1$, where $\norm{\cdot}_{\psi_2}$ is the sub-Gaussian norm defined in \Cref{sec:tools}.
\end{assumption}

Another quantity that we need to control is $S(\theta_\star, \hat g)^\top (\hat \theta - \theta_\star) = \rmD_\theta L(\theta_\star, \hat g)[\hat \theta - \theta_\star]$.
Note that $S(\theta_\star, g_0) = 0$ by the first order optimality condition.
Hence, we may control it with a smoothness assumption on the population risk.
\begin{customasmp}{3a}\label{asmp:smooth}
    For all $\theta \in \Theta_{\bar r_1}(\theta_\star)$ and $g, \bar g \in \calG_{r_2}(g_0)$, it holds that
    \begin{align*}
        \abs{\rmD_g \rmD_\theta L(\theta_\star, \bar g)[\theta - \theta_\star, g - g_0]} \le \beta_1 \norm{\theta - \theta_\star}_{H_\star} \norm{g - g_0}_{\calG}
    \end{align*}
    for some constant $\beta_1 > 0$.
\end{customasmp}
As we will show in \Cref{sec:thms}, this assumption will lead to a slow rate which scales as $O(n^{-1} + \norm{\hat g - g_0}_{\calG}^2)$.
If $S(\theta_\star, g)$ is insensitive to $g$ around $g_0$, we can obtain a faster rate $O(n^{-1} + \norm{\hat g - g_0}_{\calG}^4)$.
This insensitivity can be characterized by the \emph{Neyman orthogonality} and higher order smoothness.
\begin{customasmp}{3b}\label{asmp:orthogonal}
    The population risk $L$ is \emph{Neyman orthogonal} at $(\theta_\star, g_0)$ over $\Theta_{\bar r_1}(\theta_\star) \times \calG_{r_2}(g_0)$.
    Moreover, it holds for some constant $\beta_2 > 0$ that
    \begin{align*}
        \abs{\rmD_g^2 \rmD_\theta L(\theta_\star, \bar g)[\theta - \theta_\star, g - g_0, g - g_0]} \le \beta_2 \norm{\theta - \theta_\star}_{H_\star} \norm{g - g_0}_{\calG}^2
    \end{align*}
    for all $\theta \in \Theta_{\bar r_1}(\theta_\star)$ and $g, \bar g \in \calG_{r_2}(g_0)$.
\end{customasmp}

To facilitate the control of the empirical Hessian, we use the pseudo self-concordance as in Definition \ref{def:modify_self_concord}, which allows us to relate $H_n(\theta, g)$ and $H(\theta, g)$ to $H_n(\theta_\star, g)$ and $H(\theta_\star, g)$, respectively.

\begin{customasmp}{4}[Uniform pseudo self-concordance]\label{asmp:pseudo_self_concordance}
    For any $z \in \calZ$ and $g \in \calG_{r_2}(g_0)$, $\ell(\cdot, g; z)$ is pseudo self-concordant with parameter $R$ on $\Theta_{\bar r_1}(\theta_\star)$.
    Consequently, for any $g \in \calG_{r_2}(g_0)$, $L(\cdot, g)$ is pseudo self-concordant with parameter $R$ on $\Theta_{\bar r_1}(\theta_\star)$.
\end{customasmp}

Since $\hat \theta$ is random, we assume that the Hessian satisfies the Bernstein condition so that we can use a covering number argument to relate $H_n(\hat \theta, g)$ to $H(\theta_\star, g)$.
Due to the variability in $\hat g$, the Bernstein condition is satisfied uniformly over a neighborhood of $g_0$ and we also assume the stability of $H(\theta_\star, g)$ around $g_0$.
\begin{customasmp}{5}\label{asmp:hessian_bernstein}
    For any $\theta \in \Theta_{\bar r_1}(\theta_\star)$ and $g \in \calG_{r_2}(g_0)$, the centered sandwich Hessian
    \begin{align*}
        H(\theta, g)^{-1/2} H(\theta, g; Z) H(\theta, g)^{-1/2} - I_d
    \end{align*}
    satisfies a Bernstein condition with parameter $K_{2}$ and
    \begin{align*}
        \sigma_H^2 := \sup_{\theta \in \Theta_{\bar r_1}(\theta_\star), g \in \calG_{r_2}(g_0)} \anorm{\Var\left( H(\theta, g)^{-1/2} H(\theta, g; Z) H(\theta, g)^{-1/2} \right)}_2 < \infty,
    \end{align*}
    where, for a matrix $J \in \reals^{d \times d}$, we define $\anorm{J}_2 := \max\{\lambda_{\max}(J), \abs{\lambda_{\min}(J)}\}$ and $\Var(J) := \Expect[JJ^\top] - \Expect[J] \Expect[J]^\top$.
    Moreover, there exist constants $\kappa$ and $\mathcal{K}$ depending on $r_2$ such that
    \begin{align}\label{eq:hessian_stable}
        \kappa H_\star \preceq H(\theta_\star, g) \preceq \mathcal{K} H_\star, \quad \mbox{for all } g \in \calG_{r_2}(g_0).
    \end{align}
\end{customasmp}

\subsection{Main Results}

We now present our main results.
We will discuss our results in more detail in \Cref{sec:discussion}.
The first result is a fast rate of convergence for the excess risk assuming Neyman orthogonality.

\begin{theorem}[Fast rate]\label{thm:risk_bound}
    Under Assumptions \ref{asmp:local}, \ref{asmp:score_subg}, \ref{asmp:orthogonal}, \ref{asmp:pseudo_self_concordance}, and \ref{asmp:hessian_bernstein}, the OSL estimator $\hat \theta$ has excess risk, with probability at least $1 - \delta$,
    \begin{align}\label{eq:risk_bound}
        \erisk(\hat \theta, g_0) \lesssim \frac{e^{3Rr_1}}{\kappa^2} \left[ \frac{K_1^2 \log{(1/\delta)} \bar d_\star}{n} + \beta_2^2 \norm{\hat g - g_0}_{\calG}^4 \right]
    \end{align}
    whenever $n \ge \max\{N_{\bar r_1, r_2}(\delta/5), 16(K_{2}^2 + 2\sigma_H^2) [\log{(20d/\delta)} + d\log{(3Rr_1/\log{2})}]^2\}$,
    where $\lesssim$ hides an absolute constant.
\end{theorem}

When Neyman orthogonality fails to hold, we have a similar bound with $\norm{\hat g - g_0}_{\calG}^4$ being replaced by $\norm{\hat g - g_0}_{\calG}^2$.
\begin{theorem}[Slow rate]\label{thm:risk_bound_slow}
    Under Assumptions \ref{asmp:local}, \ref{asmp:score_subg}, \ref{asmp:smooth}, \ref{asmp:pseudo_self_concordance}, and \ref{asmp:hessian_bernstein}, the OSL estimator $\hat \theta$ has excess risk, with probability at least $1 - \delta$,
    \begin{align}\label{eq:risk_bound_slow}
        \erisk(\hat \theta, g_0) \lesssim \frac{e^{3Rr_1}}{\kappa^2} \left[ \frac{K_1^2 \log{(1/\delta)} \bar d_\star}{n} + \beta_1^2 \norm{\hat g - g_0}_{\calG}^2 \right]
    \end{align}
    whenever $n \ge \max\{N_{\bar r_1, r_2}(\delta/5), 16(K_{2}^2 + 2\sigma_H^2) [\log{(20d/\delta)} + d\log{(3Rr_1/\log{2})}]^2\}$,
    where $\lesssim$ hides an absolute constant.
\end{theorem}

The detailed proofs of Theorems \ref{thm:risk_bound} and \ref{thm:risk_bound_slow} are deferred to \Cref{sec:proof}.
On a high level, the proofs proceed as follows.
To begin with, due to \Cref{asmp:local}, we can dedicate our analysis to the case when $\hat \theta \in \Theta_{\bar r_1}(\theta_\star)$.
By Taylor's theorem,
\begin{align*}
    \erisk(\hat \theta, g_0) := L(\hat \theta, g_0) - L(\theta_\star, g_0) = S(\theta_\star, g_0)^\top (\hat \theta - \theta_\star) + \frac12 \norm{\hat \theta - \theta_\star}_{H(\bar \theta, g_0)}^2
\end{align*}
for some $\bar \theta \in \mbox{Conv}\{\hat \theta, \theta_\star\}$.
By the first order orthogonality condition, it holds that $S(\theta_\star, g_0) = 0$.
For the second term, it follows from the property of the pseudo self-concordance (\Cref{asmp:pseudo_self_concordance}) that
\begin{align*}
    \norm{\hat \theta - \theta_\star}_{H(\bar \theta, g_0)}^2 \le e^{R \norm{\bar \theta - \theta_\star}_2} \norm{\hat \theta - \theta_\star}_{H_\star}^2 \le e^{R r_1} \norm{\hat \theta - \theta_\star}_{H_\star}^2.
\end{align*}
It now remains to control $\norm{\hat \theta - \theta_\star}_{H_\star}^2$.

By Taylor's theorem again, it holds that
\begin{align}\label{eq:empirical_risk_taylor}
    L_n(\hat \theta, \hat g) - L_n(\theta_\star, \hat g) = S_n(\theta_\star, \hat g)^\top (\hat \theta - \theta_\star) + \frac12 \norm{\hat \theta - \theta_\star}^2_{H_n(\bar \theta', \hat g)},
\end{align}
where $\bar \theta' \in \mbox{Conv}\{\hat \theta, \theta_\star\}$.
By the optimality of $\hat \theta$, we have $L_n(\hat \theta, \hat g) - L_n(\theta_\star, \hat g) \le 0$.
We then lower bound the right hand side of \eqref{eq:empirical_risk_taylor}.
According to \Cref{asmp:local}, we have $\hat g \in \calG_{r_2}(g_0)$ with high probability when $n$ is sufficiently large.
The following Lemma \ref{lem:independence}, which is a direct consequence of the independence between $\{Z_i\}_{i=1}^n$ and $\hat g$, allows us to work with fixed $g \in \calG_{r_2}(g_0)$ instead of the random estimator $\hat g$.
\begin{lemma}\label{lem:independence}
    Let $\mathcal{A}(\hat g, \{Z_i\}_{i=1}^n)$ be some event regarding $\hat g$ and $\{Z_i\}_{i=1}^n$.
    Let $\calG' \subset \calG$.
    If there exists $\delta \in (0, 1)$ such that $\Prob(\calA(g, \{Z_i\}_{i=1}^n)) \ge 1 - \delta$ for all fixed $g \in \calG'$, then $\Prob(\calA(\hat g, \{Z_i\}_{i=1}^n)) \ge (1 - \delta) \Prob(\hat g \in \calG')$.
\end{lemma}

Now we focus on the score term $S_n(\theta_\star, g)^\top (\hat \theta - \theta_\star)$ in \eqref{eq:empirical_risk_taylor} with $\hat g$ replaced by a fixed $g \in \calG_{r_2}(g_0)$.
We split it into two terms
\begin{align}\label{eq:two_terms}
    [S_n(\theta_\star, g) - S(\theta_\star, g)]^\top (\hat \theta - \theta_\star) + S(\theta_\star, g)^\top (\hat \theta - \theta_\star).
\end{align}
The first term in \eqref{eq:two_terms} can be controlled using the sub-Gaussianity of the score.
Recall that
\begin{align*}
    \bar d_\star := \sup_{g \in \calG_{r_2}(g_0)} \Tr(H_\star^{-1/2} G(\theta_\star, g) H_\star^{-1/2}).
\end{align*}
\begin{proposition}\label{prop:score}
    Under \Cref{asmp:score_subg}, it holds for any fixed $g \in \calG_{r_2}(g_0)$ that, with probability at least $1 - \delta$,
    \begin{align*}
        \norm{S_n(\theta_\star, g) - S(\theta_\star, g)}_{H_\star^{-1}}^2 \lesssim \frac{K_1^2 \log{(1/\delta)} \bar d_\star}{n}.
    \end{align*}
\end{proposition}
We handle the second term in \eqref{eq:two_terms} by Neyman orthogonality and smoothness assumptions.
\begin{lemma}\label{lem:score_orthogonal}
    Under \Cref{asmp:orthogonal}, it holds that
    \begin{align*}
        S(\theta_\star, g)^\top (\theta - \theta_\star) \ge -\frac{\beta_2}{2} \norm{\theta - \theta_\star}_{H_\star} \norm{g - g_0}_{\calG}^2, \quad \mbox{for all } \theta \in \Theta_{\bar r_1}(\theta_0) \mbox{ and } g \in \calG_{r_2}(g_0).
    \end{align*}
\end{lemma}
By Proposition \ref{prop:score} and Lemma \ref{lem:score_orthogonal}, we have
\begin{align}\label{eq:score_term}
    S_n(\theta_\star, g)^\top (\hat \theta - \theta_\star)
    & \ge -\norm{S_n(\theta_\star, g) - S(\theta_\star, g)}_{H_\star^{-1}} \norm{\hat \theta - \theta_\star}_{H_\star} + S(\theta_\star, g)^\top (\hat \theta - \theta_\star) \nonumber \\
    &\gtrsim - \sqrt{\frac{K_1^2 \log{(1/\delta)} \bar d_\star}{n}} \norm{\hat \theta - \theta_\star}_{H_\star} - \beta_2 \norm{\hat \theta - \theta_\star}_{H_\star} \norm{\hat g - g_0}_{\calG}^2.
\end{align}

For the Hessian term, with $\hat g$ replaced by $g$, $\lVert \hat \theta - \theta_\star \rVert_{H_n(\bar \theta, g)}$ in \eqref{eq:empirical_risk_taylor}, we control it using pseudo self-concordance and a covering number argument.
\begin{proposition}\label{prop:hessian}
    Under Assumptions \ref{asmp:pseudo_self_concordance} and \ref{asmp:hessian_bernstein}, it holds, with probability at least $1 - \delta$, that
    \begin{align*}
        \frac{\kappa}{4 e^{Rr_1}} H_\star \preceq H_n(\theta, g) \preceq 3 \mathcal{K} e^{Rr_1} H_\star, \quad \mbox{for all } \theta \in \Theta_{\bar r_1}(\theta_\star), g \in \calG_{r_2}(g_0),
    \end{align*}
    whenever $n \ge 16(K_{2}^2 + 2\sigma_H^2) [\log{(4d/\delta)} + d\log{(3Rr_1/\log{2})}]^2$.
\end{proposition}
As a consequence of Proposition \ref{prop:hessian}, we have
\begin{align}\label{eq:hessian_term}
    \frac12 \norm{\hat \theta - \theta_\star}_{H_n(\bar \theta, g)}^2 \gtrsim \frac{\kappa}{e^{Rr_1}} \norm{\hat \theta - \theta_\star}_{H_\star}^2.
\end{align}
Putting together \eqref{eq:empirical_risk_taylor}, \eqref{eq:score_term} and \eqref{eq:hessian_term} leads to an upper bound on $\lVert \hat \theta - \theta_\star \rVert_{H_\star}$ and thus an upper bound on the excess risk $\erisk(\hat \theta, g_0)$.

When Neyman orthogonality fails to hold, we can replace Lemma \ref{lem:score_orthogonal} by the following lemma and repeat the above steps to obtain the slow rate.
\begin{lemma}\label{lem:score}
    Under \Cref{asmp:smooth}, it holds that
    \begin{align*}
        S(\theta_\star, g)^\top (\theta - \theta_\star) \ge -\beta_1 \norm{\theta - \theta_\star}_{H_\star} \norm{g - g_0}_\calG, \quad \mbox{for all } \theta \in \Theta_{\bar r_1}(\theta_0) \mbox{ and } g \in \calG_{r_2}(g_0).
    \end{align*}
\end{lemma}

%% file: sections/sec3.tex
\begin{table}[t]
    \caption{In its simplest version (e.g., ignoring the effect of $\hat g$), our bound scales as $O(d_\star / n)$ where $d_* := \Tr(H_*^{-1/2} G_* H_*^{-1/2})$ is the effective dimension, while the bound of \cite{foster2020orthogonal} scales as $O(d'/n)$ where $d' := d^2/\lambda_{\min}(H_\star)$.
    We compare the two in different regimes of eigendecays of $G_*$ and $H_*$ assuming they share the same eigenvectors.}
    \label{tab:decay}
    \centering
    \renewcommand{\arraystretch}{1.2}
    \begin{tabular}{lccccc}
        \addlinespace[0.4em]
        \toprule
        & \multicolumn{2}{c}{\textbf{Eigendecay}} & \multicolumn{2}{c}{\textbf{Dimension Dependency}} & \textbf{Ratio} \\
        & $G_\star$ & $H_\star$ & $d_\star$ & $d'$ & $d' / d_\star$ \\
        \midrule
        Poly-Poly & $i^{-\alpha}$ & $i^{-\beta}$ & $d^{(\beta - \alpha + 1) \vee 0}$ & $d^{\beta+2}$ & $d^{(\alpha+1)\wedge(\beta+2)}$ \\
        Poly-Exp & $i^{-\alpha}$ & $ e^{-\nu i}$ & $d^{-(\alpha-1) \vee 1} e^{\nu d}$ & $d^2 e^{\nu d}$ & $d^{1 \wedge (3 - \alpha)}$ \\
        Exp-Poly & $e^{-\mu i}$ & $i^{-\beta}$ & $1$ & $d^{\beta+2}$ & $d^{\beta +2}$ \\
        Exp-Exp & $e^{-\mu i}$ & $e^{-\nu i}$ & \begin{tabular}{@{}c@{}} \qquad $d$ \mbox{ if }  $\mu = \nu$ \\ \qquad $1$ \mbox{ if } $\mu > \nu$ \\ $e^{(\nu - \mu) d}$ \mbox{ if } $\mu < \nu$ \end{tabular} & $d^2 e^{\nu d}$ & \begin{tabular}{@{}c@{}} $d e^{\nu d}$ \mbox{ if }  $\mu = \nu$ \\ $d^2 e^{\nu d}$ \mbox{ if } $\mu > \nu$ \\ $d^2 e^{\mu d}$ \mbox{ if } $\mu < \nu$ \end{tabular} \\
        \bottomrule
    \end{tabular}
\end{table}

\paragraph{Convergence rate and effective dimension.}
There are two terms in the bounds \eqref{eq:risk_bound} and \eqref{eq:risk_bound_slow}.
In the case of no nuisance parameter, the second term involving $\norm{\hat g - g}_\calG$ vanishes.
As for the first term, the profile effective dimension $\bar d_\star$ simplifies to $d_\star := \Tr(H_\star^{-1/2} G_\star H_\star^{-1/2})$ as shown in \Cref{sec:known_nuisance}.
This coincides with the result from \cite{ostrovskii2021finite} on generalized linear models, i.e., the loss is given by $\ell(\theta; Z) := \ell(Y, X^\top \theta)$.
Under a well-specified model, the effective dimension $d_\star$ becomes $d$, recovering the same rate $O(d/n)$ as in classical parametric least-squares regression \cite[see, e.g.~][Proposition 3.5]{bach2021learning}.
When the model is misspecified, the effective dimension measures the mismatch between the covariance matrix $G_\star$ and the Hessian matrix $H_\star$.
This quantity is related to the \textit{sandwich covariance} in statistics~\cite[Sec. 6.7]{wakefield2013bayesian}.

To facilitate its understanding, we summarize the effective dimension $d_\star$ in~\Cref{tab:decay} under different regimes of eigendecay, assuming that $G_\star$ and $H_\star$ share the same eigenvectors.~\Cref{tab:decay} shows that the dimension dependence can be better than $O(d)$ when the spectrum of $G_\star$ decays faster than the one of $H_\star$.
In particular, it is most favorable when the spectrum of $G_\star$ decays as $e^{-\mu i}$ and the one of $H_\star$ decays as $i^{-\beta}$.

In the case when the nuisance parameter needs to be estimated, we pay the price of not knowing the true nuisance in both of the two terms.
In this first term, we have $\bar d_\star$ rather than $d_\star$ which is the maximum effective dimension in a neighborhood of $g_0$.
As for the second term, the estimator $\hat g$ will typically have a rate of convergence $\norm{\hat g - g_0}_{\calG} = O(n^{-\varphi})$ with $\varphi < 1/2$ in high dimensions~\cite[Section 1]{chernozhukov2018double}. As a result, the term $\norm{\hat g - g_0}_\calG^2$ has a dominating effect in the bound \eqref{eq:risk_bound_slow} which is slower than $O(n^{-1})$.
If Neyman orthogonality holds as assumed in \Cref{thm:risk_bound}, we do not pay this price in the fast rate \eqref{eq:risk_bound} as long as $\varphi \ge 1/4$.
Note that Neyman orthogonality is only used in Lemma \ref{lem:score_orthogonal} to control $S(\theta_\star, g)^\top (\theta - \theta_\star)$.
If $\lvert \rmD_g \rmD_\theta L(\theta_\star, g_0)[\hat \theta - \theta_\star, \hat g - g_0] \rvert$ does not vanish but decays as $O(r_n)$, then the second term in \eqref{eq:risk_bound} will read $O(r_n^2 + \norm{\hat g - g_0}_\calG^4)$.

\paragraph{Orthogonal statistical learning and double machine learning.}
Our work lies in the framework of orthogonal statistical learning.
Under a strong convexity assumption and a Neyman orthogonality assumption on the population risk, \citet{foster2020orthogonal} obtain the rate
\begin{align}\label{eq:exist_res}
    \erisk(\hat \theta, g_0) \lesssim O\left( \frac{d^2}{n\lambda^2} + \frac{d}{\lambda^2}\norm{\hat g - g_0}_\calG^4 \right), \quad \mbox{for all } n \ge 1,
\end{align}
where $\lambda$ is the infimum of $\lambda_{\min}(H(\theta, g_0))$ over a neighborhood of $\theta_\star$ \cite[see][Theorems 1 and 3]{foster2020orthogonal}. Our results improve on theirs in several ways. When $\bar d_\star$ is at most proportional to the dimension $d$, our results improve the excess risk bound by at least a factor of $d$.
Our bounds also remove the explicit dependence on the minimum eigenvalue $\lambda$, owing to our tail assumptions \ref{asmp:score_subg} and \ref{asmp:hessian_bernstein} on the normalized score and the Hessian.
However, our bounds may depend on $\lambda$ implicitly through, e.g., the sub-Gaussian parameter $K_1$.
This dependency contributes at most a factor of $\lambda^{-1}$ for applications considered in \Cref{sec:application}.
Hence, to be more concrete, we compare $d_\star$ with $d^2 / \lambda_{\min}(H_\star)$ in different eigendecay regimes in~\Cref{tab:decay}.
For instance, when the spectrum of $G_\star$ decays as $e^{-\mu i}$ and the one of $H_\star$ decays as $i^{-\beta}$, our bound gives a rate $O(n^{-1})$ while theirs gives a rate $O(d^{\beta + 2}/n)$.

\citet{chernozhukov2018double} recently proposed a set of methods based on Neyman orthogonal scores and cross-fitting, denoted by  double or debiased machine learning (DML), to the classical problem of semi-parametric inference. There is an abundant literature on semi-parametric estimation in mathematical statistics~\citep{xia2006semi,wellner2007two} and machine learning~\citep{smola1998semiparametric,rakotomamonjy2005frames,mackey2018orthogonal,bertail2021learning} and we refer to classical books for a bibliography~\citep{bickel1993efficient,ruppert2003semiparametric,tsiatis2006semiparametric,kosorok2008introduction,van2011targeted}.
In~\cite{chernozhukov2018double}, the authors establish the asymptotic normality of their estimators when the dimension of the target parameter is kept fixed. In this work, we provide \emph{non-asymptotic} guarantees in terms of excess risk for DML under self-concordance, allowing the dimension of the target parameter to grow at the rate $d = O(n^{1/2})$. In a recent work~\citep{nekipelov2022regularised}, regularized estimators with sparsity-inducing regularization are analyzed in terms of parameter recovery under restricted convexity assumptions.

%% file: sections/sec4.tex
\subsection{Treatment Effect Estimation}

Let us revisit the problem of treatment effect estimation under the assumption of unconfoundedness, as presented in the introduction. Before we had a binary treatment case, and our target of inference was a one dimensional parameter. Here, to better fit our framework, we  consider a vector of predictors $D := (D^k)_{k=1}^d \in \mathbb{R}^d$, under partially linear CEF of the following form:
$$
\Expect[Y \mid D, X ] = \theta_0^\top D + \gamma_0(X).
$$

Note that, by targeting multiple coefficients $\theta \in \Theta \subset \mathbb{R}^d$, we can model not only multiple treatments,
but also heterogeneous treatment effects across different binary groups,
as well as other non-linear effects, by performing nonlinear transformations of our original treatment variable.
To illustrate, suppose $T$ is the original treatment and there is a finite dimensional feature map $D := \phi(T) =  [\phi_1(T), \dots, \phi_d(T)]$ such that $\Expect[Y \mid T, X] = \theta_0^\top \phi(T) + \gamma_{0}(X)$. Under the the assumption of uncounfoundedness conditional on $X$, the ATE of setting $T$ to $t_1$ versus $t_0$ is then given by: 
$$
\Expect[Y(T = t_1) - Y(T = t_0)] = \theta_0^\top (\phi(t_1) - \phi(t_0)).
$$
Heterogeneous effects could be estimated in a similar manner. Letting $G = [G_{1}, \dots, G_{d}]$ denote indicators for $d$ subgroups, and letting $T \in \{0,1\}$ denote the binary treatment indicator, we can define the covariates $D := TG$. With this flexibility in mind, we now examine the partially linear model in the context of our framework.

\paragraph{Multiple target coefficients in a partially linear model.}
Let the ``target'' predictors be $D := (D^k)_{k=1}^d \in \mathbb{R}^d$.
Consider the model
\begin{align*}
    D &= \alpha_0(X) + U \\
    Y &= \theta_0^\top D + \gamma_0(X) + V = \zeta_0(X) + \theta_0^\top U + V,
\end{align*}
where $\alpha_0 : \reals^p \rightarrow \reals^d$, $\Expect[U \mid X] \txtover{a.s.}{=} 0$ and $\Expect[V \mid D, X] \txtover{a.s.}{=} 0$ are the residuals.
Moreover, $U$ has a non-singular covariance $\Sigma_u$ and $V$ is independent of $D$ and $X$ with variance $\sigma_v^2 > 0$.
We reparametrize the model by $g = (\zeta, \alpha)$ and work with the loss
\begin{align*}
    \ell(\theta, g; Z) := [Y - \zeta(X) - \theta^\top (D - \alpha(X))]^2.
\end{align*}
Since $\Expect[UV] = \Expect[(D - \alpha_0(X)) V] = 0$,
we have
\begin{align*}
    L(\theta, g)
    &:= \Expect[[Y - \zeta(X) - \theta^\top (D - \alpha(X))]^2] \\
    &= \Expect\left[ \left( \zeta_0(X) - \zeta(X) - \theta^\top (\alpha_0(X) - \alpha(X)) \right)^2 \right] + \norm{\theta - \theta_0}_{\Sigma_u}^2 + \sigma_v^2.
\end{align*}
This implies that the population risk $L$ at $g_0$ has a unique minimizer $\theta_\star = \theta_0$.

Now suppose that $U$ is bounded (i.e., $\norm{U}_2 \le M$), $V$ is sub-Gaussian with parameter $\norm{V}_{\psi_2}$, and $\norm{\cdot}_{\calG}$ is chosen as the sup-norm, i.e., $\norm{g}_{\calG} = \sup_{x} \sqrt{\norm{\alpha(x)}^2 + \zeta^2(x)}$.
Let us verify the assumptions in \Cref{sub:asmp} for this model.
For \Cref{asmp:score_subg}, we have
\begin{align*}
    S(\theta_\star, g; Z)
    &= 2(\alpha_0(X) - \alpha(X) + U)\left[ (\alpha_0(X) - \alpha(X))^\top \theta_\star - (\zeta_0(X) - \zeta(X) + V) \right], \\
    S(\theta_\star, g)
    &= 2\Expect\left[ (\alpha_0(X) - \alpha(X)) (\alpha_0(X) - \alpha(X))^\top \theta_\star - (\alpha_0(X) - \alpha(X))(\zeta_0(X) - \zeta(X)) \right],
\end{align*}
and $G(\theta_\star, g) \succeq 4\sigma_v^2 \Sigma_u$.
Note that $\norm{\alpha_0(X) - \alpha(X)} \le \norm{g - g_0}_\calG \le r_2$, $\norm{\zeta_0(X) - \zeta(X)} \le \norm{g - g_0}_\calG \le r_2$, $\norm{U}_2 \le M$, and $V$ is sub-Gaussian.
Hence, it follows from Lemmas \ref{lem:bounded_rv},\ref{lem:prod_bounded_subg} and \ref{lem:matrix_vec_subg} that the normalized score is sub-Gaussian with sub-Gaussian norm
\begin{align*}
    K_1 \lesssim \frac{(r_2 + M) [r_2(\norm{\theta_\star}_2 + 1) + \norm{V}_{\psi_2}]}{\sigma_v \sqrt{\lambda_{\min}(\Sigma_u)}}.
\end{align*}

For \Cref{asmp:orthogonal}, it holds that, for any $\bar \theta \in \Theta$,
\begin{align*}
    \rmD_g \rmD_\theta L(\bar \theta, g_0)[\theta - \theta_\star, g - g_0] \equiv 0,
\end{align*}
which verifies the Neyman orthogonality.
Moreover, we have $H_\star = \nabla_\theta^2 L(\theta_\star, g_0) = 2 \Sigma_u$ and
\begin{align*}
    &\quad \abs{\rmD_g^2 \rmD_\theta L(\theta_\star, \bar g)[\theta - \theta_\star, g - g_0, g - g_0]} \\
    &= \abs{4(\theta - \theta_\star)^\top \Expect\left\{ (\alpha_0(X) - \alpha(X)) \left[ (\alpha_0(X) - \alpha(X))^\top \theta_\star - (\zeta_0(X) - \zeta(X)) \right] \right\}} \\
    &\lesssim \frac{\norm{\theta_\star}_2 + 1}{\sqrt{\lambda_{\min}(\Sigma_u)}} \norm{\theta - \theta_\star}_{H_\star} \norm{g - g_0}_{\calG}^2.
\end{align*}
In other words, \Cref{asmp:orthogonal} holds true with $\beta_2 \lesssim (\norm{\theta_\star}_2 + 1)/\sqrt{\lambda_{\min}(\Sigma_u)}$.

For \cref{asmp:pseudo_self_concordance}, both the loss $\ell$ and the population risk $L$ are pseudo self-concordant with arbitrary parameter $R \ge 0$ since their third derivatives w.r.t.~$\theta$ are zero.
For \cref{asmp:hessian_bernstein}, we have
\begin{align*}
    H(\theta, g; Z) &= 2(\alpha_0(X) - \alpha(X) + U) (\alpha_0(X) - \alpha(X) + U)^\top, \\
    H(\theta, g) &= 2\Expect[(\alpha_0(X) - \alpha(X) + U) (\alpha_0(X) - \alpha(X) + U)^\top] \succeq 2\Sigma_u.
\end{align*}
Note that $H(\theta, g)^{-1/2} H(\theta, g; Z) H(\theta, g)^{-1/2} - I_d$ has mean-zero and satisfies
\begin{align*}
    \anorm{H(\theta, g)^{-1/2} H(\theta, g; Z) H(\theta, g)^{-1/2} - I_d}_2
    \lesssim \norm{H(\theta, g)^{-1}}_2 [\norm{g - g_0}_\calG^2 + \norm{U}_2^2] \lesssim \frac{r_2^2 + M^2}{\lambda_{\min}(\Sigma_u)}.
\end{align*}
Hence, it follows from \citet[Equation 6.30]{wainwright2019high} that
\begin{align*}
    H(\theta, g)^{-1/2} H(\theta, g; Z) H(\theta, g)^{-1/2} - I_d
\end{align*}
satisfies the Bernstein condition with parameter $K_2 \lesssim (r_2^2 + M^2)/\lambda_{\min}(\Sigma_u)$.
Moreover, $\sigma_H^2 \lesssim (r_2^2 + M^2)^2/\lambda_{\min}^2(\Sigma_u)$.
For the stability \eqref{eq:hessian_stable}, we have $H_\star = 2 \Sigma_u$ and
\begin{align*}
    2\Sigma_u \preceq H(\theta_\star, g) = 2\Expect[(\alpha_0(X) - \alpha(X) + U) (\alpha_0(X) - \alpha(X) + U)^\top] \preceq 2[r_2^2 I_d + \Sigma_u].
\end{align*}
Thus, the stability holds with $\kappa = 1$ and $\calK = 1 + r_2^2 / \lambda_{\min}(\Sigma_u)$.
To summarize, invoking \Cref{thm:risk_bound} gives the following risk bound up to a constant factor:
\begin{align}
    \frac{(r_2 + M)^2[r_2(\norm{\theta_\star}_2 + 1) + \norm{V}_{\psi_2}]^2}{\sigma_v^2 \lambda_{\min}(\Sigma_u)} \frac{\bar d_\star}{n} \log{(1/\delta)} + \frac{(\norm{\theta_\star}_2 + 1)^2}{\lambda_{\min}(\Sigma_u)} \norm{\hat g - g_0}_{\calG}^4.
\end{align}
\begin{remark}
    As a comparison, assuming $\norm{U}_2 \le M$, $\norm{V}_2 \le M'$, and $R := \sup_{\theta \in \Theta} \norm{\theta}_2 \vee 1 < \infty$, Theorems 1 and 3 of \citet{foster2020orthogonal} yield the bound
    \begin{align*}
        \frac{\bar K^2}{\lambda_{\min}(\Sigma_u)^2} \frac{d^2}{n} \log{(1/\delta)} + \frac{R\bar K d}{\lambda_{\min}(\Sigma_u)^2} \norm{\hat g - g_0}_{\calG}^4,
    \end{align*}
    where $\bar K := (r_2 + M)[r_2(R+1) + M' + RM]$.
    Our result not only requires less stringent assumptions but also improves their result by a factor of $d/\lambda_{\min}(\Sigma_u)$ when $\bar d_\star \lesssim d$.
\end{remark}

\subsection{Semi-Parametric Logistic Regression}
We consider a semi-parametric logistic regression model to illustrate the usefulness of the pseudo self-concordance assumption.

Let $Z := (X, W, Y)$ where $X \in \reals^d$, $W \in \mathcal{W}$, and $Y \in \{-1, 1\}$.
Consider the model
\begin{align*}
    \Prob(Y = 1 \mid X, W) = \sigma\left( \theta_0^\top X + g_0(W) \right),
\end{align*}
where $\sigma(u) := (1 + e^{-u})^{-1}$.
It is clear that
\begin{align*}
    \Expect[Y \mid X, W]
    &= \Prob(Y = 1 \mid X, W) - \Prob(Y = -1 \mid X, W)
    = 2\sigma(\theta_0^\top X + g_0(W)) - 1.
\end{align*}
The logistic loss is defined as
\begin{align*}
    \ell(\theta, g; Z) := \log{\left( 1 + \exp\left(-Y(\theta^\top X + g(W)) \right) \right)}.
\end{align*}
It can be shown that
\begin{align*}
    S(\theta, g; Z) = \left[ \sigma(\theta^\top X + g(W)) - \frac12 - \frac{Y}2 \right]X
\end{align*}
and
\begin{align*}
    S(\theta, g) &= \Expect[\Expect[S(\theta, g; Z) \mid X, W]] = \Expect\left\{ X \left[ \sigma(\theta^\top X + g(W)) - \sigma(\theta_0^\top X + g_0(W)) \right] \right\} \\
    H(\theta, g) &= \Expect\left[ \sigma(\theta^\top X + g(W))[1 - \sigma(\theta^\top X + g(W))] XX^\top \right] \\
    G(\theta, g) &= \Expect\left\{ [ \sigma(\theta^\top X + g(W)) - \sigma(\theta_0^\top X + g_0(W)) ]^2 XX^\top \right\} - S(\theta, g) S(\theta, g)^\top + H(\theta_0, g_0).
\end{align*}
Assume that $H_\star := H(\theta_0, g_0)$ is non-singular.
The population risk $L(\theta, g_0)$ is minimized at $\theta_\star = \theta_0$.

Suppose that $X$ is bounded (i.e., $\norm{X}_2 \le M$), $r_1 \le \lambda_{\min}(H_\star)/M^3$, and $r_2 \le \lambda_{\min}(H_\star)/M^2$.
By the non-singularity of $H_\star$, the covariance $G(\theta, g)$ is non-singular for all $\theta \in \Theta$ and $g \in \calG$.
Let us verify the assumptions in \Cref{sub:asmp} for this model.
For \Cref{asmp:score_subg}, it follows directly from Lemmas \ref{lem:bounded_rv} and \ref{lem:matrix_vec_subg} that the normalized score is sub-Gaussian with sub-Gaussian norm $K_1 \lesssim M / \sqrt{\lambda_{\min}(H_\star)}$.
\Cref{asmp:smooth} holds true with $\beta_1 := M/(4\sqrt{\lambda_{\min}(H_\star)})$ since
\begin{align*}
    &\quad \abs{\rmD_g \rmD_\theta L(\theta_\star, \bar g)[\theta - \theta_\star, g - g_0]} \\
    &= \abs{\Expect\left[ \sigma(\theta_\star^\top X + \bar g(W))[1 - \sigma(\theta_\star^\top X + \bar g(W))] X^\top (\theta - \theta_\star) (g(W) - g_0(W)) \right]} \\
    &\le \frac{M}{4\sqrt{\lambda_{\min}(H_\star)}} \norm{\theta - \theta_\star}_{H_\star} \norm{g - g_0}_{\calG}.
\end{align*}
For \Cref{asmp:pseudo_self_concordance}, we have, with $a := \theta^\top x + g(w)$,
\begin{align*}
    \abs{\rmD_\theta^3 \ell(\theta, g; z)[u, u, u]}
    &= \abs{\sigma(a)[1 - \sigma(a)][1 - 2\sigma(a)] (u^\top x)^3} \\
    &\le \abs{\sigma(a)[1 - \sigma(a)] \norm{u}_2 \norm{x}_2 (u^\top x)^2} \\
    &\le M \norm{u}_2 \rmD_\theta^2 \ell(\theta, g; z)[u, u],
\end{align*}
which implies that $\ell(\cdot, g; z)$ is pseudo self-concordance with parameter $R = M$.
The pseudo self-concordance of $L(\theta, g)$ can be verified similarly.

For \Cref{asmp:hessian_bernstein}, we first show that $H(\theta, g)$ is non-singular on $\Theta_{\bar r_1}(\theta_\star) \times \calG_{r_2}(g_0)$.
In fact, with $A := \theta^\top X + g(W)$ and $A_0 := \theta_0^\top X + g_0(W)$, we have
\begin{align*}
    \anorm{H(\theta, g) - H_\star}_2
    &\le \Expect\left[ \abs{\sigma(A) - \sigma(A_0)} \anorm{XX^\top}_2 \right] \\
    &\le \frac14 \Expect\left[ [\lvert\theta^\top X - \theta_0^\top X\rvert + \abs{g(W) - g_0(W)}] \anorm{XX^\top}_2 \right] \\
    &\le \frac14 \left( r_1 M^3 + r_2 M^2 \right), \quad \mbox{for all } (\theta, g) \in \Theta_{\bar r_1}(\theta_\star) \times \calG_{r_2}(g_0).
\end{align*}
This yields that
\begin{align*}
    \anorm{H_\star^{-1/2} H(\theta,g) H_\star^{-1/2} - I_d}_2 \le \frac{1}{4\lambda_{\min}(H_\star)}\left( r_1 M^3 + r_2 M^2 \right) \le \frac12,
\end{align*}
and thus $H(\theta, g) \succeq I_d/2$ for all $(\theta, g) \in \Theta_{\bar r_1}(\theta_\star) \times \calG_{r_2}(g_0)$.
Analogously, we can show that
\begin{align*}
    \anorm{H_\star^{-1/2} H(\theta_\star, g) H_\star^{-1/2} - I_d}_2 \le \frac{r_2 M^2}{4\lambda_{\min}(H_\star)} \le \frac14,
\end{align*}
and thus the stability \eqref{eq:hessian_stable} holds true with $\kappa = 3/4$ and $\mathcal{K} = 5/4$.
As for the Bernstein condition, we note that
\begin{align*}
    \anorm{H(\theta, g; Z)}_2 = \anorm{\sigma(\theta^\top X + g(W))[1 - \sigma(\theta^\top X + g(W))] XX^\top}_2 \le \frac{M^2}{4}.
\end{align*}
It follows that
\begin{align*}
    \anorm{H(\theta, g)^{-1/2} H(\theta, g; Z) H(\theta, g)^{-1/2} - I_d}_2 \le \frac{M^2}{2 \lambda_{\min}(H(\theta, g))} \le M^2.
\end{align*}
Due to \cite[Equation 6.30]{wainwright2019high}, $H(\theta, g)^{-1/2} H(\theta, g; Z) H(\theta, g)^{-1/2} - I_d$ satisfies the Bernstein condition with parameter $K_2 \lesssim M^2$.
Moreover, $\sigma_H^2 \lesssim M^4$.
To summarize, invoking \Cref{thm:risk_bound_slow} gives the following risk bound up to a constant factor:
\begin{align*}
    \frac{M^2}{\lambda_{\min}(H_\star)} \left[ \frac{\bar d_\star}{n} \log{(1/\delta)} + \norm{\hat g - g_0}_{\calG}^2 \right].
\end{align*}

\begin{remark}
    Since the semi-parametric logistic loss does not satisfy the Neyman orthogonality, the results of \cite{foster2020orthogonal} do not directly apply here.
\end{remark}

%% file: sections/sec5.tex
We established non-asymptotic guarantees in terms of the excess risk for the orthogonal statistical learning under pseudo self-concordance, allowing the dimension of the target parameter to grow at the rate $d = O(n^{1/2})$.
The dimension dependency in our bound is characterized by the effective dimension---the trace of the sandwich covariance matrix---which recovers existing results in supervised learning without the nuisance parameter.
Compared with previous work \citep{foster2020orthogonal},
our results improve on the excess risk bound at least by a factor of $d$ in a wide range of eigendecay regimes. The extension of our theoretical analysis to handle sparse regularization is an interesting venue for future work.

%% file: sections/known_nuisance.tex
In this section, we assume that the true nuisance parameter $g_0$ is known and control the excess risk.
The proofs in this section are inspired by and extend those from~\cite{ostrovskii2021finite}.
We denote by $\hat \theta_0$ the minimizer of the empirical risk $L_n(\theta, g_0)$.
Our analysis is local to $\theta_\star$, in other words, we make the following assumption on $\hat \theta_0$.
Recall that $\Theta_{r}(\theta_\star) := \{\theta \in \Theta: \norm{\theta - \theta_\star}_{H_\star} < r\}$.

\begin{customasmp}{6}\label{asmp:local_known_nuisance}
    Let $r_0 > 0$ be a constant and $\tilde r_0 := \min\{r_0, r_0 \sqrt{\lambda_{\min}(H_\star)}\} > 0$.
    There exists a function $N_{\tilde r_0}: [0, 1] \rightarrow \bbN_+$ such that for any $\delta \in (0, 1)$ we have, with probability at least $1 - \delta$, $\hat \theta_0 \in \Theta_{\tilde r_0}(\theta_\star)$ for all $n \ge N_{\tilde r_0}(\delta)$.
\end{customasmp}

\paragraph{Control of the score.}
In order to control the score, we assume that the normalized score at $\theta_\star$ is sub-Gaussian.
\begin{customasmp}{7}\label{asmp:subg_score_known_nuisance}
    The normalized score at $\theta_0$ is sub-Gaussian, i.e., there exists a constant $K_{1,0} > 0$ such that
    \begin{equation*}
    \norm{G_\star^{-1/2}[S(\theta_\star, g_0; Z) - S(\theta_\star, g_0)]}_{\psi_2} \le K_{1,0} .
    \end{equation*}
\end{customasmp}

Recall that $d_\star := \Tr(\Omega_\star)$ and $\Omega_\star := H_\star^{-1/2} G_\star H_\star^{-1/2}$.
\begin{proposition}\label{prop:score_known_nuisance}
    Under \Cref{asmp:subg_score_known_nuisance}, it holds that, with probability at least $1 - \delta$,
    \begin{align*}
        \norm{S_n(\theta_\star, g_0)}_{H_\star^{-1}}^{2} \lesssim \frac1n[d_\star + K_{1,0}^2 \log{(e / \delta)} \norm{\Omega_\star}_2],
    \end{align*}
    where $\lesssim$ hides an absolute constant.
\end{proposition}
\begin{proof}
    By the first order optimality condition, we have $S(\theta_\star, g_0) = 0$.
    As a result,
    \begin{align*}
        X := \sqrt{n} G_\star^{-1/2} S_n(\theta_\star, g_0; Z)
    \end{align*}
    is an isotropic random vector.
    Moreover, we have $\norm{X}_{\psi_2} \lesssim K_{1,0}$ by Lemma \ref{lem:sum_subg}.
    Define $J := G_\star^{1/2} H_\star^{-1} G_\star^{1/2} / n$.
    Then we have
    \begin{align*}
        \norm{S_n(\theta_\star, g_0)}_{H_\star^{-1}}^{2} = \norm{X}_{J}^2.
    \end{align*}
    Invoking \Cref{thm:isotropic_tail} yields the claim.
\end{proof}

\paragraph{Control of the Hessian.}
In order to control the Hessian, we use the pseudo self-concordance as in Definition \ref{def:modify_self_concord}.

\begin{customasmp}{8}\label{asmp:self_concord}
    For any $z \in \calZ$, $\ell(\theta, g_0; z)$ is pseudo self-concordant on $\Theta_{\tilde r_0}(\theta_\star)$, i.e.,
    \begin{align*}
        \abs{\rmD_\theta^3 \ell(\theta, g_0; z)[u, u, u]} \le R_0 \norm{u}_2 \rmD_\theta^2 \ell(\theta, g_0; z)[u, u], \quad \mbox{for all } \theta \in \Theta_{\tilde r_0}(\theta_\star), u \in \reals^d.
    \end{align*}
    Moreover, $L(\theta, g_0)$ is pseudo self-concordant on $\Theta_{\tilde r_0}(\theta_\star)$.
\end{customasmp}

We also assume that the Hessian $H(\theta, g_0; Z)$ satisfies the Bernstein condition uniformly over $\Theta_{\tilde r_0}(\theta_\star)$.
\begin{customasmp}{9}\label{asmp:unif_subg_hessian_known_nuisance}
    For any $\theta \in \Theta_{\tilde r_0}(\theta_\star)$, the centered sandwich Hessian
    \begin{align*}
        H(\theta, g_0)^{-1/2} H(\theta, g_0; Z) H(\theta, g_0)^{-1/2} - I_d
    \end{align*}
    satisfies a Bernstein condition with parameter $K_{2,0}$.
    Moreover,
    \begin{align*}
        \sigma_{H,0}^2 := \sup_{\theta \in \Theta_{\tilde r_0}(\theta_\star)} \anorm{\Var(H(\theta, g_0)^{-1/2} H(\theta, g_0; Z) H(\theta, g_0)^{-1/2})}_2 < \infty.
    \end{align*}
\end{customasmp}

\begin{proposition}\label{prop:H_self_concord}
    Under \Cref{asmp:self_concord}, for any $\theta \in \Theta_{\tilde r_0}(\theta_\star)$, we have
    \begin{align}\label{eq:hessian_known_nuisance}
        e^{-R_0 r_0} H_\star \preceq H(\theta, g_0) \preceq e^{R_0 r_0} H_\star.
    \end{align}
    Moreover, if \Cref{asmp:unif_subg_hessian_known_nuisance} holds true, then, with probability at least $1 - \delta$, we have
    \begin{align}\label{eq:emp_hessian_known_nuisance}
        \frac{1}{4e^{R_0 r_0}} H_\star \preceq H_n(\theta, g_0) \preceq 3 e^{R_0 r_0} H_\star, \quad \mbox{for all } \theta \in \Theta_{\tilde r_0}(\theta_\star),
    \end{align}
    whenever $n \ge 16 (K_{2,0}^2 + 2\sigma_{H,0}^2) \left[ \log{(4d/\delta)} + d\log{(3r_0 R_0 / \log{2})} \right]^2$.
\end{proposition}

\begin{proof}
    According to \Cref{asmp:self_concord} and Proposition \ref{prop:pseudo_self_concordance}, we have
    \begin{align*}
        e^{-R_0 \norm{\theta - \theta_\star}_2} H_\star \preceq H(\theta, g_0) \preceq e^{R_0 \norm{\theta - \theta_\star}_2} H_\star.
    \end{align*}
    Hence, the claim \eqref{eq:hessian_known_nuisance} follows from $\norm{\theta - \theta_\star}_2 \le \tilde r_0 / \sqrt{\lambda_{\min}(H_\star)} \le r_0$.
    As for \eqref{eq:emp_hessian_known_nuisance}, we prove it in the following steps.

    \emph{Step 1.} Let $\epsilon = \sqrt{\lambda_{\min}(H_\star)} \log{2} / R_0$. Take an $\epsilon$-covering $\calN_\epsilon$ of $\Theta_{\tilde r_0}(\theta_\star)$ w.r.t.~$\norm{\cdot}_{H_\star}$, and let $\pi(\theta)$ be the projection of $\theta$ onto $\calN_\epsilon$.
    By the self-concordance of $\ell(\cdot, g_0; z)$ (\Cref{asmp:self_concord}), we have, for all $\theta \in \Theta_{\tilde r_0}(\theta_\star)$,
    \begin{align*}
        e^{-R_0 r} H(\pi(\theta), g_0; Z) \preceq H(\theta, g_0; Z) \preceq e^{R_0 r} H(\pi(\theta), g_0; Z),
    \end{align*}
    where $r := \norm{\theta - \pi(\theta)}_2 \le \epsilon / \sqrt{\lambda_{\min}(H_\star)} = \log{2}/R_0$.
    It then follows that
    \begin{align*}
        \frac12 H(\pi(\theta), g_0; Z_i) \preceq H(\theta, g_0; Z_i) \preceq 2 H(\pi(\theta), g_0; Z_i), \quad \mbox{for all } \theta \in \Theta_{\tilde r_0}(\theta_\star) \mbox{ and } i \in [n],
    \end{align*}
    which yields
    \begin{align}\label{eq:hessian_proj_known_nuisance}
        \frac12 H_n(\pi(\theta), g_0) \preceq H_n(\theta, g_0) \preceq 2H_n(\pi(\theta), g_0), \quad \mbox{for all } \theta \in \Theta_{\tilde r_0}(\theta_\star).
    \end{align}

    \emph{Step 2.} By \Cref{thm:bernstein_matrix}, for each $\theta \in \Theta_{\tilde r_0}(\theta_\star)$, it holds that, with probability at least $1 - \delta$,
    \begin{align*}
        \anorm{H(\theta, g_0)^{-1/2} H_n(\theta, g_0) H(\theta, g_0)^{-1/2} - I_d}_2 \le \frac{1}{2} ,
    \end{align*}
    or equivalently,
    \begin{align*}
        \frac12 H(\theta, g_0) \preceq H_n(\theta, g_0) \preceq \frac{3}{2} H(\theta, g_0)
    \end{align*}
    whenever $n \ge 16 (K_{2,0}^2 + 2\sigma_{H,0}^2) \log^2{(2d/\delta)}$.
    Since $\abs{\calN_\epsilon} \le (3 \tilde r_0/\epsilon)^d$ \citep{ostrovskii2021finite},
    by a union bound, we get, with probability at least $1 - \delta/2$,
    \begin{align}\label{eq:hessin_union_known_nuisance}
        \frac12 H(\pi(\theta), g_0) \preceq H_n(\pi(\theta), g_0) \preceq \frac{3}{2} H(\pi(\theta), g_0), \quad \mbox{for all } \theta \in \Theta_{\tilde r_0}(\theta_\star),
    \end{align}
    whenever $n \ge 16 (K_{2,0}^2 + 2\sigma_{H,0}^2) \left[ \log{(4d/\delta)} + d\log{(3r_0 R_0 / \log{2})} \right]^2$.
    Hence, the statement \eqref{eq:emp_hessian_known_nuisance} follows from \eqref{eq:hessian_known_nuisance}, \eqref{eq:hessian_proj_known_nuisance}, and \eqref{eq:hessin_union_known_nuisance}.
\end{proof}

\paragraph{Control of the excess risk.}
The next theorem shows that the excess risk is upper bounded by $d_\star/n$ up to a constant factor.

\begin{theorem}\label{thm:risk_known_nuisance}
    Under Assumptions \ref{asmp:local_known_nuisance}-\ref{asmp:unif_subg_hessian_known_nuisance}, with probability at least $1 - \delta$, the excess risk of $\hat \theta_0$ satisfies
    \begin{align}\label{eq:risk_known_nuisance}
        \erisk(\hat \theta_0, g_0) \lesssim K_{1,0}^2 e^{3R_0 r_0} \log{(1 / \delta)} \frac{d_\star}{n}
    \end{align}
    whenever $n \ge \max\{N_{\tilde r_0}(\delta/3), 16 (K_{2,0}^2 + 2\sigma_{H,0}^2) \left[ \log{(12d/\delta)} + d\log{(3r_0 R_0 / \log{2})} \right]^2\}$.
\end{theorem}

\begin{proof}
    We start by defining three events.
    Let
    \begin{align*}
        \calA &:= \left\{ \hat \theta_0 \in \Theta_{\tilde r_0}(\theta_\star) \right\} \\
        \calB &:= \left\{ \frac{1}{4e^{R_0r_0}} H_\star \preceq H_n(\theta, g_0) \preceq 3e^{R_0r_0} H_\star, \quad \mbox{for all } \theta \in \Theta_{\tilde r_0}(\theta_\star) \right\} \\
        \calC &:= \left\{ \norm{S_n(\theta_\star, g_0)}_{H_\star^{-1}} \lesssim \sqrt{\frac{d_\star + K_{1,0}^2 \log{(3e/\delta)} \norm{\Omega_\star}_2}{n}} \right\}.
    \end{align*}
    In the following, we let
    \begin{align*}
        n \ge \max\{N_{\tilde r_0}(\delta/3), 16 (K_{2,0}^2 + 2\sigma_{H,0}^2) \left[ \log{(12d/\delta)} + d\log{(3r_0 R_0 / \log{2})} \right]^2\}.
    \end{align*}
    According to \Cref{asmp:local_known_nuisance}, we have $\Prob(\calA) \ge 1 - \delta/3$.
    By Proposition \ref{prop:H_self_concord}, it holds that $\Prob(\calB) \ge 1 - \delta/3$.
    Finally, it follows from Proposition \ref{prop:score_known_nuisance} that $\Prob(\calC) \ge 1 - \delta/3$.
    
    Now, we prove the upper bound \eqref{eq:risk_known_nuisance} on the event $\calA \calB \calC$.
    By Taylor's theorem,
    \begin{align*}
        \erisk(\hat \theta_0, g_0) := L(\hat \theta_0, g_0) - L(\theta_\star, g_0) = S(\theta_\star, g_0)^\top (\hat \theta_0 - \theta_\star) + \frac{1}{2} \norm{\hat \theta_0 - \theta_\star}_{H(\bar \theta, g_0)}^2
    \end{align*}
    for some $\bar \theta \in \mbox{Conv}\{\hat \theta_0, \theta_\star\} \subset \Theta_{\tilde r_0}(\theta_\star)$.
    According to \eqref{eq:hessian_known_nuisance}, it holds that
    \begin{align*}
        \norm{\hat \theta_0 - \theta_\star}_{H(\bar \theta, g_0)}^2 \le e^{R_0 r_0} \norm{\hat \theta_0 - \theta_\star}_{H_\star}^2.
    \end{align*}
    By the first order optimality condition, we have $S(\theta_\star, g_0) = 0$.
    As a result,
    \begin{align*}
        \erisk(\hat \theta_0, g_0) \le \frac{1}{2}e^{R_0 r_0} \norm{\hat \theta_0 - \theta_\star}_{H_\star}^2.
    \end{align*}
    It then suffices to upper bound $\norm{\hat \theta_0 - \theta_\star}_{H_\star}$.
    
    Note that, by Taylor's theorem,
    \begin{align*}
        L_n(\hat \theta_0, g_0) - L_n(\theta_\star, g_0) = S_n(\theta_\star, g_0)^\top (\hat \theta_0 - \theta_\star) + \frac12 \norm{\hat \theta_0 - \theta_\star}_{H_n(\bar \theta, g_0)}^2
    \end{align*}
    for some $\bar \theta \in \mbox{Conv}\{\hat \theta_0, \theta_\star\} \subset \Theta_{\tilde r_0}(\theta_\star)$.
    On the event $\calB$, it holds that
    \begin{align*}
        \norm{\hat \theta_0 - \theta_\star}_{H_n(\bar \theta, g_0)}^2 \ge \frac{1}{4e^{R_0 r_0}} \norm{\hat \theta_0 - \theta_\star}_{H_\star}^2.
    \end{align*}
    Moreover, by the Cauchy-Schwarz inequality,
    \begin{align*}
        S_n(\theta_\star, g_0)^\top (\hat \theta_0 - \theta_\star) \ge - \norm{S_n(\theta_\star, g_0)}_{H_\star^{-1}} \norm{\hat \theta_0 - \theta_\star}_{H_\star}.
    \end{align*}
    On the event $\calC$, we get
    \begin{align*}
        \norm{S_n(\theta_\star, g_0)}_{H_\star^{-1}} \lesssim \frac{\sqrt{d_\star + K_{1,0}^2 \log{(3e / \delta)} \norm{\Omega_\star}_2}}{\sqrt{n}}.
    \end{align*}
    Due to the optimality of $\hat \theta_0$, we also have $L_n(\hat \theta_0, g_0) - L_n(\theta_\star, g_0) \le 0$.
    Consequently,
    \begin{align*}
        \frac{1}{4e^{R_0 r_0}} \norm{\hat \theta_0 - \theta_\star}_{H_\star}^2 \lesssim \frac{\sqrt{d_\star + K_{1,0}^2 \log{(3e / \delta)} \norm{\Omega_\star}_2}}{\sqrt{n}} \norm{\hat \theta_0 - \theta_\star}_{H_\star}.
    \end{align*}
    It then follows that
    \begin{align*}
        \erisk(\hat \theta_0, g_0) \le \frac{e^{R_0 r_0}}{2} \norm{\hat \theta_0 - \theta_\star}_{H_\star}^2 \lesssim \frac{d_\star + K_{1,0}^2 \log{(3e / \delta)} \norm{\Omega_\star}_2}{e^{-3R_0 r_0} n} \lesssim K_{1,0}^2 e^{3R_0 r_0} \log{(1 / \delta)} \frac{d_\star}{n}.
    \end{align*}
    Therefore, the claim \eqref{eq:risk_known_nuisance} holds with probability at least
    \begin{align*}
        \Prob(\calA \calB \calC) = 1 - \Prob(\calA^c \cup \calB^c \cup \calC^c) \ge 1 - \Prob(\calA^c) - \Prob(\calB^c) - \Prob(\calC^c) \ge 1 - \delta.
    \end{align*}
\end{proof}

\begin{remark}
    Our results generalize the results~\cite[Theorem 4.1]{ostrovskii2021finite} which were developed for parametric linear models, i.e., when the loss is given by $\ell(\theta; Z) := \ell(Y, X^\top \theta)$. The paper of~\cite{ostrovskii2021finite} relies heavily on the special structure of the Hessian.
    Our results apply to a broader class of models owing to the matrix Bernstein inequality.
\end{remark}

%% file: sections/proof.tex
We then consider the case when the true nuisance parameter $g_0$ is unknown and estimated from a separate sample as in the OSL meta-algorithm.
Again, our analysis is local to both $\theta_\star$ and $g_0$.
The independence between $\hat g$ and the sample $\{Z_i\}_{i=1}^n$ greatly simplifies our analysis according to Lemma \ref{lem:independence} in \Cref{sec:thms}.

\begin{proof}[Proof of Lemma \ref{lem:independence}]
    By the independence between $\hat g$ and $\{Z_i\}_{i=1}^n$, we have
    \begin{align*}
        \Expect[\ind\{\calA(\hat g, \{Z_i\}_{i=1}^n)\} \mid \hat g](g) = \Expect[\ind\{\calA(g, \{Z_i\}_{i=1}^n)\}] \ge 1 - \delta, \quad \mbox{for any } g \in \calG'.
    \end{align*}
    By the tower property of the conditional expectation,
    \begin{align*}
        \Prob(\calA(\hat g, \{Z_i\}_{i=1}^n)) &= \Expect\big[ \Expect[\ind\{\calA(\hat g, \{Z_i\}_{i=1}^n)\} \mid \hat g] \big] \\
        &\ge \Expect\big[ \Expect[\ind\{\calA(\hat g, \{Z_i\}_{i=1}^n)\} \mid \hat g] \ind\{\hat g \in \calG'\} \big] \ge (1 - \delta) \Prob(\hat g \in \calG').
    \end{align*}
\end{proof}

\paragraph{Control of the score.}
Recall that $\bar d_\star := \sup_{g \in \calG_{r_2}(g_0)} \Tr(H_\star^{-1/2} G(\theta_\star, g) H_\star^{-1/2})$.
\begin{proof}[Proof of Proposition \ref{prop:score}]
    Define $W := \sqrt{n} G(\theta_\star, g)^{-1/2} \bar S_n(\theta_\star, g)$ where
    \begin{align*}
        \bar S_n(\theta_\star, g) := S_n(\theta_\star, g) - S(\theta_\star, g).
    \end{align*}
    It is straightforward to check that $W$ is isotropic.
    Moreover, it follows from Lemma \ref{lem:sum_subg} that $\norm{W}_{\psi_2} \le K_1$.
    Let $J := G(\theta_\star, g)^{1/2} H_\star^{-1} G(\theta_\star, g)^{1/2}$.
    By \Cref{thm:isotropic_tail}, we have, with probability at least $1 - \delta$,
    \begin{align*}
        \norm{W}_J^2 \lesssim K_1^2 \log{(e/\delta)} \Tr(J) \le K_1^2 \log{(e/\delta)} \bar d_\star.
    \end{align*}
    The statement then follows from the fact that
    \begin{align*}
        \norm{W}_J^2 = W^\top J W = n \bar S_n(\theta_\star, g)^\top H_\star^{-1} \bar S_n(\theta_\star, g) = n \norm{S_n(\theta_\star, g) - S(\theta_\star, g)}_{H_\star^{-1}}^2.
    \end{align*}
\end{proof}

\begin{proof}[Proof of Lemma \ref{lem:score_orthogonal}]
    By Taylor's theorem,
    \begin{align*}
        S(\theta_\star, g)^\top (\hat \theta - \theta_\star) &= \rmD_\theta L(\theta_\star, g)[\hat \theta - \theta_\star] \\
        &= \rmD_\theta L(\theta_\star, g_0)[\hat \theta - \theta_0] + \rmD_g \rmD_\theta L(\theta_\star, g_0)[\hat \theta - \theta_\star, g - g_0]\; + \\
        &\qquad \frac12 \rmD_g^2 \rmD_\theta L(\theta_\star, \bar g)[\hat \theta - \theta_\star, g - g_0, g - g_0] \\
        &\ge -\frac{\beta_2}{2} \norm{\hat \theta - \theta_\star}_{H_\star} \norm{g - g_0}_{\calG}^2,
    \end{align*}
    where the last inequality follows from the first order optimality condition and \Cref{asmp:orthogonal}.
\end{proof}

\begin{proof}[Proof of Lemma \ref{lem:score}]
    By Taylor's theorem,
    \begin{align*}
        S(\theta_\star, g)^\top (\hat \theta - \theta_\star) &= \rmD_\theta L(\theta_\star, g)[\hat \theta - \theta_\star] \\
        &= \rmD_\theta L(\theta_\star, g_0)[\hat \theta - \theta_0] + \rmD_g \rmD_\theta L(\theta_\star, \bar g)[\hat \theta - \theta_\star, g - g_0] \\
        &\ge -\beta_1 \norm{\hat \theta - \theta_\star}_{H_\star} \norm{g - g_0}_{\calG},
    \end{align*}
    where the last inequality follows from the first order optimality condition and \Cref{asmp:smooth}.
\end{proof}

\paragraph{Control of the Hessian.}
We then prove Proposition \ref{prop:hessian}.

\begin{proof}[Proof Proposition \ref{prop:hessian}]
    Fix an arbitrary $g \in \calG_{r_2}(g_0)$.
    \emph{Step 1.}
    Invoking Proposition \ref{prop:pseudo_self_concordance} leads to
    \begin{align*}
        e^{-Rr} H(\theta_\star, g) \preceq H(\theta, g) \preceq e^{R r} H(\theta_\star, g), \quad \mbox{for all } \theta \in \Theta_{\bar r_1}(\theta_\star),
    \end{align*}
    where $r := \norm{\theta - \theta_\star}_{2} \le \bar r_1 / \sqrt{\lambda_{\min}(H_\star)} \le r_1$.
    Consequently, by \eqref{eq:hessian_stable}, we have, for all $\theta \in \Theta_{\tilde 
    r_1}(\theta_\star)$,
    \begin{align}\label{eq:hessian}
        \kappa e^{-Rr_1} H_\star \preceq e^{-Rr_1} H(\theta_\star, g) \preceq H(\theta, g) \preceq e^{Rr_1} H(\theta_\star, g) \preceq \calK e^{Rr_1} H_\star.
    \end{align}

    \emph{Step 2.} Let $\epsilon := \sqrt{\lambda_{\min}(H_\star)} \log{2} / R$. Take an $\epsilon$-covering $\calN_\epsilon$ of $\Theta_{\bar r_1}(\theta_\star)$ w.r.t.~$\norm{\cdot}_{H_\star}$, and let $\pi(\theta)$ be the projection of $\theta$ onto $\calN_\epsilon$.
    By \Cref{asmp:pseudo_self_concordance}, we have, for all $\theta \in \Theta_{\bar r_1}(\theta_\star)$,
    \begin{align*}
        e^{-Rr'} H(\pi(\theta), g; Z) \preceq H(\theta, g; Z) \preceq e^{R r'} H(\pi(\theta), g; Z),
    \end{align*}
    where $r' := \norm{\theta - \pi(\theta)}_2 \le \epsilon / \sqrt{\lambda_{\min}(H_\star)} = \log{2}/R$.
    This implies that
    \begin{align*}
        \frac12 H(\pi(\theta), g; Z_i) \preceq H_n(\theta, g; Z_i) \preceq 2H_n(\pi(\theta), g; Z_i), \quad \mbox{for all } \theta \in \Theta_{\bar r_1}(\theta_\star) \mbox{ and } i \in [n].
    \end{align*}
    Hence,
    \begin{align}\label{eq:hessian_proj}
        \frac12 H_n(\pi(\theta), g) \preceq H_n(\theta, g) \preceq 2 H_n(\pi(\theta), g), \quad \mbox{for all } \theta \in \Theta_{\bar r_1}(\theta_\star).
    \end{align}

    \emph{Step 3.} By \Cref{thm:bernstein_matrix}, for each $\theta \in \Theta_{\bar r_1}(\theta_\star)$, it holds that, with probability at least $1 - \delta$,
    \begin{align*}
        \frac12 H(\theta, g) \preceq H_n(\theta, g) \preceq \frac{3}{2} H(\theta, g)
    \end{align*}
    whenever $n \ge 16(K_{2}^2 + 2\sigma_H^2) \log^2{(2d/\delta)}$.
    Since $\abs{\calN_\epsilon} \le (3\bar r_1/\epsilon)^d$ \citep{ostrovskii2021finite},
    by a union bound, we get, with probability at least $1 - \delta/2$,
    \begin{align}\label{eq:hessin_union}
        \frac12 H(\pi(\theta), g) \preceq H_n(\pi(\theta), g) \preceq \frac{3}{2} H(\pi(\theta), g), \quad \mbox{for all } \theta \in \Theta_{\bar r_1}(\theta_\star),
    \end{align}
    whenever $n \ge 16(K_{2}^2 + 2\sigma_H^2) [\log{(4d/\delta)} + d\log{(3Rr_1/\log{2})}]^2$.
    Hence, the claim follows from \eqref{eq:hessian}, \eqref{eq:hessian_proj}, and \eqref{eq:hessin_union}.
\end{proof}

\paragraph{Control of the excess risk.}

Now we are ready to prove \Cref{thm:risk_bound}.

\begin{proof}[Proof of \Cref{thm:risk_bound}]
    Fix an arbitrary $g \in \calG_{r_2}(g_0)$.
    We start by defining three events.
    Let
    \begin{align*}
        \calA &:= \left\{ \hat \theta \in \Theta_{\bar r_1}(\theta_\star), \hat g \in \calG_{r_2}(g_0) \right\} \\
        \calB(g) &:= \left\{ \frac{\kappa}{4e^{Rr_1}} H_\star \preceq H_n(\theta, g) \preceq 3 \calK e^{Rr_1} H_\star, \quad \mbox{for all } \theta \in \Theta_{\bar r_1}(\theta_\star) \right\} \\
        \calC(g) &:= \left\{ \norm{S_n(\theta_\star, g) - S(\theta_\star, g)}_{H_\star^{-1}} \lesssim \sqrt{\frac{K_{1}^2 \log{(5e/\delta)} \bar d_\star}{n}} \right\}.
    \end{align*}
    In the following, we let
    \begin{align*}
        n \ge \max\{N_{\bar r_1, r_2}(\delta/5), 16(K_{2}^2 + 2\sigma_H^2) [\log{(20d/\delta)} + d\log{(3Rr_1/\log{2})}]^2\}.
    \end{align*}
    According to \Cref{asmp:local}, we have $\Prob(\calA) \ge 1 - \delta/5$.
    By Propositions \ref{prop:score} and \ref{prop:hessian}, it holds that $\Prob(\calC(g)) \ge 1 - \delta/5$ and $\Prob(\calB(g)) \ge 1 - \delta/5$, respectively.
    Since $g \in \calG_{r_2}(g_0)$ is arbitrary, it follows from Lemma \ref{lem:independence} that
    \begin{align*}
        \Prob(\calB(\hat g)) \ge (1 - \delta/5) \Prob(\hat g \in \calG_{r_2}(g_0)) \ge 1 - 2\delta / 5.
    \end{align*}
    Similarly, $\Prob(\calC(\hat g)) \ge 1 - 2\delta / 5$.
    
    Now, we prove the upper bound \eqref{eq:risk_known_nuisance} on the event $\calA \calB(\hat g) \calC(\hat g)$.
    By Taylor's theorem,
    \begin{align*}
        \erisk(\hat \theta, g_0) := L(\hat \theta, g_0) - L(\theta_\star, g_0) = S(\theta_\star, g_0)^\top (\hat \theta - \theta_\star) + \frac{1}{2} \norm{\hat \theta - \theta_\star}_{H(\bar \theta, g_0)}^2
    \end{align*}
    for some $\bar \theta \in \mbox{Conv}\{\hat \theta, \theta_\star\} \subset \Theta_{r_1}(\theta_\star)$.
    According to \eqref{eq:hessian_known_nuisance}, it holds that
    \begin{align*}
        \norm{\hat \theta - \theta_\star}_{H(\bar \theta, g_0)}^2 \le e^{Rr_1} \norm{\hat \theta - \theta_\star}_{H_\star}^2.
    \end{align*}
    By the first order optimality condition, we have $S(\theta_\star, g_0) = 0$.
    As a result,
    \begin{align*}
        \erisk(\hat \theta, g_0) \le \frac{e^{Rr_1}}{2} \norm{\hat \theta - \theta_\star}_{H_\star}^2.
    \end{align*}
    It then suffices to upper bound $\norm{\hat \theta - \theta_\star}_{H_\star}$.
    
    Note that, by Taylor's theorem,
    \begin{align*}
        L_n(\hat \theta, \hat g) - L_n(\theta_\star, \hat g) = S_n(\theta_\star, \hat g)^\top (\hat \theta - \theta_\star) + \frac12 \norm{\hat \theta - \theta_\star}_{H_n(\bar \theta, \hat g)}^2
    \end{align*}
    for some $\bar \theta \in \mbox{Conv}\{\hat \theta, \theta_\star\} \subset \Theta_{\bar r_1}(\theta_\star)$.
    On the event $\calB(\hat g)$, it holds that
    \begin{align*}
        \norm{\hat \theta - \theta_\star}_{H_n(\bar \theta, \hat g)}^2 \ge \frac{\kappa}{4e^{Rr_1}} \norm{\hat \theta - \theta_\star}_{H_\star}^2.
    \end{align*}
    Moreover,
    \begin{align*}
        &\quad S_n(\theta_\star, \hat g)^\top (\hat \theta - \theta_\star) \\
        &= \bar S_n(\theta_\star, \hat g)^\top (\hat \theta - \theta_\star) + S(\theta_\star, \hat g)^\top (\hat \theta - \theta_\star) \\
        &\ge \bar S_n(\theta_\star, \hat g)^\top (\hat \theta - \theta_\star) - \frac{\beta_2}{2} \norm{\hat \theta - \theta_\star}_{H_\star} \norm{\hat g - g_0}_{\calG}^2, \quad \mbox{by Lemma \ref{lem:score_orthogonal}} \\
        &\ge -\norm{\bar S_n(\theta_\star, \hat g)}_{H_\star^{-1}} \norm{\hat \theta - \theta_\star}_{H_\star} - \frac{\beta_2}{2} \norm{\hat \theta - \theta_\star}_{H_\star} \norm{\hat g - g_0}_{\calG}^2, \quad \mbox{by the Cauchy-Schwarz} \\
        &\gtrsim -\norm{\hat \theta - \theta_\star}_{H_\star} \left[ \sqrt{\frac{K_1^2 \log{(5e/\delta)} \bar d_\star}{n}} + \frac{\beta_2}{2} \norm{\hat g - g_0}_{\calG}^2 \right], \quad \mbox{by the event } \calC(\hat g).
    \end{align*}
    Due to the optimality of $\hat \theta$, we also have $L_n(\hat \theta, \hat g) - L_n(\theta_\star, \hat g) \le 0$.
    Consequently,
    \begin{align*}
        \frac{\kappa}{4 e^{Rr_1}} \norm{\hat \theta - \theta_\star}_{H_\star}^2 \lesssim \left[ \sqrt{\frac{K_1^2 \log{(5e/\delta)} \bar d_\star}{n}} + \frac{\beta_2}{2} \norm{\hat g - g_0}_{\calG}^2 \right] \norm{\hat \theta - \theta_\star}_{H_\star}.
    \end{align*}
    It then follows that
    \begin{align*}
        \erisk(\hat \theta, g_0)
        &\le \frac{e^{Rr_1}}{2} \norm{\hat \theta - \theta_\star}_{H_\star}^2 \\
        &\lesssim \frac{e^{3Rr_1}}{\kappa^2} \left[ \frac{K_1^2 \log{(1/\delta)} \bar d_\star}{n} + \beta_2^2 \norm{\hat g - g_0}_{\calG}^4 \right].
    \end{align*}
    Therefore, the claim holds with probability at least
    \begin{align*}
        \Prob(\calA \calB(\hat g) \calC(\hat g)) = 1 - \Prob(\calA^c \cup \calB(\hat g)^c \cup \calC(\hat g)^c) \ge 1 - \Prob(\calA^c) - \Prob(\calB(\hat g)^c) - \Prob(\calC(\hat g)^c) \ge 1 - \delta.
    \end{align*}
\end{proof}

\begin{proof}[Proof of \Cref{thm:risk_bound_slow}]
    The proof is largely similar to the one of \Cref{thm:risk_bound}.
    The only difference is that Lemma \ref{lem:score_orthogonal} is replaced by Lemma \ref{lem:score} at places where it is invoked.
\end{proof}

%% file: sections/tools.tex
In this section, we give the precise definitions of sub-Gaussian random vectors \cite[Chapter 3.4]{vershynin2018high} and the matrix Bernstein condition \cite[Chapter 6.4]{wainwright2019high}.
We then review and prove some key results that are used in our analysis.
Finally, we recall a proposition for the pseudo self-concordance.

\begin{definition}[Sub-Gaussian vector]\label{def:subg_vec}
    Let $S \in \reals^d$ be a mean-zero random vector.
    We say $S$ is sub-Gaussian if $\ip{S, s}$ is sub-Gaussian for every $s \in \reals^d$.
    Moreover, we define the sub-Gaussian norm of $S$ as
    \begin{align*}
        \norm{S}_{\psi_2} := \sup_{\norm{s}_2 = 1} \norm{\ip{S, s}}_{\psi_2}.
    \end{align*}
    Note that $\norm{\cdot}_{\psi_2}$ is a norm and satisfies, e.g., the triangle inequality.
\end{definition}

\begin{definition}[Matrix Bernstein condition]\label{def:matrix_bernstein}
    Let $H \in \reals^{d \times d}$ be a zero-mean symmetric random matrix.
    We say $H$ satisfies a Bernstein condition with parameter $b > 0$ if, for all $j \ge 3$,
    \begin{align*}
        \Expect[H^j] \preceq \frac12 j! b^{j-2} \Var(H).
    \end{align*}
\end{definition}

\begin{lemma}\label{lem:bounded_rv}
    If $X \in \reals^d$ is a bounded random vector with $\norm{X}_2 \le_{a.s.} M < \infty$, then $X$ is a sub-Gaussian random vector with
    \begin{align*}
        \norm{X}_{\psi_2} \lesssim M \quad \mbox{and} \quad \norm{X - \Expect[X]}_{\psi_2} \lesssim M.
    \end{align*}
\end{lemma}
\begin{proof}
    For any $x \in \reals^d$, we have $\ip{X, x} \le \norm{X}_2 \norm{x} \le M$.
    Hence, by definition, $X$ is a sub-Gaussian random vector with
    \begin{align*}
        \norm{X}_{\psi_2} = \sup_{\norm{x} = 1} \norm{\ip{X, x}}_{\psi_2} \lesssim M.
    \end{align*}
    Moreover, since $\norm{\cdot}_{\psi_2}$ is a norm, we have
    \begin{align*}
        \norm{X - \Expect[X]}_{\psi_2} \le \norm{X}_{\psi_2} + \norm{\Expect[X]}_{\psi_2}.
    \end{align*}
    Note that $\norm{a}_{\psi_2} \lesssim \norm{a}_2$ for a constant vector $a$.
    Hence, we have $\norm{X - \Expect[X]}_{\psi_2} \lesssim M$.
\end{proof}

\begin{lemma}\label{lem:prod_bounded_subg}
    Let $X \in \reals^d$ be a random vector with $\norm{X}_2 \le_{a.s.} M < \infty$ and $Y \in \reals$ be a sub-Gaussian random variable.
    Then $X Y$ is a sub-Gaussian random vector with
    \begin{align*}
        \norm{X Y}_{\psi_2} \le M \norm{Y}_{\psi_2}.
    \end{align*}
\end{lemma}
\begin{proof}
    By the definition of sub-Gaussian random variables, we have
    \begin{align}\label{eq:subg_vec}
        \Expect[\exp(Y^2 / \norm{Y}_{\psi_2}^2)] \le 2.
    \end{align}
    It then follows that, for any $\norm{x} = 1$,
    \begin{align*}
        \Expect[\exp((x^\top X)^2 Y^2 / M^2 \norm{Y}_{\psi_2}^2)] 
        &\le \Expect[\exp( \norm{X}_2^2 Y^2 / M^2 \norm{Y}_{\psi_2}^2)] \le 2,
    \end{align*}
    and thus $\norm{X Y}_{\psi_2} \le M \norm{Y}_{\psi_2}$.
\end{proof}

\begin{lemma}\label{lem:matrix_vec_subg}
    Let $X \in \reals^d$ be a sub-Gaussian random vector and $A \in \reals^{d \times d}$ be a fixed matrix.
    Then $AX$ is a sub-Gaussian random vector with
    \begin{align*}
        \norm{AX}_{\psi_2} \le \norm{A}_2 \norm{X}_{\psi_2}.
    \end{align*}
\end{lemma}
\begin{proof}
    Take an arbitrary $\norm{x}_2 = 1$.
    It holds that
    \begin{align*}
        \Expect[\exp(\lambda x^\top AX)] 
        &= \Expect\left[\exp\left(\lambda \norm{A^\top x}_2 \left(\frac{A^\top x}{\norm{A^\top x}_2}\right)^\top X \right)\right] \\
        &\le \exp\left( \norm{A^\top x}_2^2 \norm{X}_{\psi_2}^2 \lambda^2 \right), \quad \mbox{by the sub-Gaussianity of } X \\
        &\le \exp\left( \norm{A}_2^2 \norm{X}_{\psi_2}^2 \lambda^2 \right).
    \end{align*}
    Hence, we obtain $\norm{AX}_{\psi_2} \le \norm{A}_2 \norm{X}_{\psi_2}$.
\end{proof}

The sum of i.i.d.~sub-Gaussian vectors is also sub-Gaussian according to the following lemma.
\begin{lemma}[\cite{vershynin2018high}, Lemma 5.9]\label{lem:sum_subg}
Let $X_1, \dots, X_n$ be i.i.d.~random vectors, then we have $\norm{\sum_{i=1}^n X_i}_{\psi_2}^2 \lesssim \sum_{i=1}^n \norm{X_i}_{\psi_2}^2$.
\end{lemma}

We call a random vector $X \in \reals^d$ isotropic if $\Expect[X] = 0$ and $\Expect[XX  ^\top] = I_d$.
The following theorem is a tail bound for quadratic forms of isotropic sub-Gaussian random vectors.
\begin{theorem}[\cite{ostrovskii2021finite}, Theorem A.1]\label{thm:isotropic_tail}
Let $X \in \reals^d$ be an isotropic random vector with $\norm{X}_{\psi_2} \le K$, and let $J \in \reals^{d \times d}$ be positive semi-definite.
Then, with probability at least $1 - \delta$, it holds that
\begin{align}
  \norm{X}_{J}^2 - \text{Tr}(J) \lesssim K^2\left( \norm{J}_2 \sqrt{\log{(e/\delta)}} + \norm{J}_{\infty} \log{(1/\delta)} \right).
\end{align}
\end{theorem}

A zero-mean symmetric random matrix $Q$ is said to be sub-Gaussian with parameter $V$ if $\Expect[e^{\lambda Q}] \preceq e^{\lambda^2 V / 2}$ for all $\lambda \in \reals$.
The next theorem is the Bernstein bound for random matrices.

\begin{theorem}[\cite{wainwright2019high}, Theorem 6.17]\label{thm:bernstein_matrix}
Let $\{Q_i\}_{i=1}^n$ be a sequence of zero-mean independent symmetric random matrices that satisfies the Bernstein condition with parameter $b > 0$.
Then, for all $\delta > 0$, it holds that
\begin{align}
  \Prob\left( \anorm{\frac1n \sum_{i=1}^n Q_i}_2 \ge \delta \right) \le 2 \Rank\left(\sum_{i=1}^n \Var(Q_i)\right) \exp\left\{ -\frac{n\delta^2}{2(\sigma^2 + b\delta)} \right\},
\end{align}
where $\sigma^2 := \frac1n \anorm{\sum_{i=1}^n \Var(Q_i)}_2$.
\end{theorem}

One advantage of pseudo self-concordance is that we can relate the Hessian at $y$ to the Hessian at $x$ in terms of the norm $\norm{y - x}_2$.
\begin{proposition}[\cite{bach2010self}, Proposition 1]\label{prop:pseudo_self_concordance}
    Assume that $f$ is pseudo self-concordant with parameter $R$ on $\calX$.
    For any $y \in \calX$, we have
    \begin{align*}
        e^{-R \norm{y - x}_2} \nabla^2 f(x) \preceq \nabla^2 f(y) \preceq e^{R \norm{y - x}_2} \nabla^2 f(x).
    \end{align*}
\end{proposition}

%% file: main-arxiv.bbl
\begin{thebibliography}{31}
\providecommand{\natexlab}[1]{#1}
\providecommand{\url}[1]{\texttt{#1}}
\expandafter\ifx\csname urlstyle\endcsname\relax
  \providecommand{\doi}[1]{doi: #1}\else
  \providecommand{\doi}{doi: \begingroup \urlstyle{rm}\Url}\fi

\bibitem[Bach(2010)]{bach2010self}
F.~Bach.
\newblock Self-concordant analysis for logistic regression.
\newblock \emph{Electronic Journal of Statistics}, 4, 2010.

\bibitem[Bach(2021)]{bach2021learning}
F.~Bach.
\newblock \emph{Learning Theory from First Principles}.
\newblock Online version, 2021.

\bibitem[Bertail et~al.(2021)Bertail, Cl{\'e}men{\c{c}}on, Guyonvarch, and
  Noiry]{bertail2021learning}
P.~Bertail, S.~Cl{\'e}men{\c{c}}on, Y.~Guyonvarch, and N.~Noiry.
\newblock Learning from biased data: A semi-parametric approach.
\newblock In \emph{ICML}, 2021.

\bibitem[Bickel et~al.(1998)Bickel, Klaassen, Ritov, and
  Wellner]{bickel1993efficient}
P.~J. Bickel, C.~A. Klaassen, Y.~Ritov, and J.~A. Wellner.
\newblock \emph{Efficient and Adaptive Estimation for Semiparametric Models}.
\newblock Springer, 1998.

\bibitem[Chernozhukov et~al.(2018)Chernozhukov, Chetverikov, Demirer, Duflo,
  Hansen, Newey, and Robins]{chernozhukov2018double}
V.~Chernozhukov, D.~Chetverikov, M.~Demirer, E.~Duflo, C.~Hansen, W.~Newey, and
  J.~Robins.
\newblock Double/debiased machine learning for treatment and structural
  parameters.
\newblock \emph{The Econometrics Journal}, 21\penalty0 (1), 2018.

\bibitem[Foster and Syrgkanis(2020)]{foster2020orthogonal}
D.~J. Foster and V.~Syrgkanis.
\newblock Orthogonal statistical learning.
\newblock \emph{arXiv preprint}, 2020.

\bibitem[Hern{\'a}n and Robins(2020)]{hernan2010causal}
M.~A. Hern{\'a}n and J.~M. Robins.
\newblock \emph{Causal Inference: What If}.
\newblock Boca Raton: Chapman \& Hall/CRC, 2020.

\bibitem[Imbens and Rubin(2015)]{imbens2015causal}
G.~W. Imbens and D.~B. Rubin.
\newblock \emph{Causal Inference in Statistics, Social, and Biomedical
  Sciences}.
\newblock Cambridge University Press, 2015.

\bibitem[Kosorok(2008)]{kosorok2008introduction}
M.~R. Kosorok.
\newblock \emph{Introduction to Empirical Processes and Semiparametric
  Inference.}
\newblock Springer, 2008.

\bibitem[Liu et~al.(2021)Liu, Zhang, and Zhou]{liu2021double}
M.~Liu, Y.~Zhang, and D.~Zhou.
\newblock Double/debiased machine learning for logistic partially linear model.
\newblock \emph{The Econometrics Journal}, 24\penalty0 (3), 2021.

\bibitem[Mackey et~al.(2018)Mackey, Syrgkanis, and Zadik]{mackey2018orthogonal}
L.~Mackey, V.~Syrgkanis, and I.~Zadik.
\newblock Orthogonal machine learning: Power and limitations.
\newblock In \emph{ICML}, 2018.

\bibitem[Murphy and Van~der Vaart(2000)]{murphy2000profile}
S.~A. Murphy and A.~W. Van~der Vaart.
\newblock On profile likelihood.
\newblock \emph{Journal of the American Statistical Association}, 95\penalty0
  (450), 2000.

\bibitem[Nekipelov et~al.(2022)Nekipelov, Semenova, and
  Syrgkanis]{nekipelov2022regularised}
D.~Nekipelov, V.~Semenova, and V.~Syrgkanis.
\newblock Regularised orthogonal machine learning for nonlinear semiparametric
  models.
\newblock \emph{The Econometrics Journal}, 25\penalty0 (1), 2022.

\bibitem[Nesterov and Nemirovskii(1994)]{yurii1994interior}
Y.~Nesterov and A.~Nemirovskii.
\newblock \emph{Interior-Point Polynomial Algorithms in Convex Programming}.
\newblock Society for Industrial and Applied Mathematics, 1994.

\bibitem[Neyman(1959)]{neyman1959optimal}
J.~Neyman.
\newblock Optimal asymptotic tests of composite hypotheses.
\newblock \emph{Probability and Statistics}, 1959.

\bibitem[Neyman(1979)]{neyman1979test}
J.~Neyman.
\newblock $c(\alpha)$ tests and their use.
\newblock \emph{Sankhy\={a}: The Indian Journal of Statistics, Series A},
  41\penalty0 (1/2), 1979.

\bibitem[Ostrovskii and Bach(2021)]{ostrovskii2021finite}
D.~M. Ostrovskii and F.~Bach.
\newblock Finite-sample analysis of {$M$}-estimators using self-concordance.
\newblock \emph{Electronic Journal of Statistics}, 15\penalty0 (1), 2021.

\bibitem[Pearl(2009)]{pearl2009causality}
J.~Pearl.
\newblock \emph{Causality}.
\newblock Cambridge University Press, 2009.

\bibitem[Peters et~al.(2017)Peters, Janzing, and
  Sch{\"o}lkopf]{peters2017elements}
J.~Peters, D.~Janzing, and B.~Sch{\"o}lkopf.
\newblock \emph{Elements of Causal Inference: Foundations and Learning
  Algorithms}.
\newblock The MIT Press, 2017.

\bibitem[Rakotomamonjy et~al.(2005)Rakotomamonjy, Canu, and
  Smola]{rakotomamonjy2005frames}
A.~Rakotomamonjy, S.~Canu, and A.~Smola.
\newblock Frames, reproducing kernels, regularization and learning.
\newblock \emph{Journal of Machine Learning Research}, 6\penalty0 (9), 2005.

\bibitem[Rosenbaum and Rubin(1983)]{Rosenbaum1983}
P.~R. Rosenbaum and D.~B. Rubin.
\newblock The central role of the propensity score in observational studies for
  causal effects.
\newblock \emph{Biometrika}, 70\penalty0 (1), 1983.

\bibitem[Ruppert et~al.(2003)Ruppert, Wand, and
  Carroll]{ruppert2003semiparametric}
D.~Ruppert, M.~P. Wand, and R.~J. Carroll.
\newblock \emph{Semiparametric Regression}.
\newblock Cambridge University Press, 2003.

\bibitem[Shpitser et~al.(2010)Shpitser, VanderWeele, and
  Robins]{shpitser:apa2012}
I.~Shpitser, T.~VanderWeele, and J.~M. Robins.
\newblock On the validity of covariate adjustment for estimating causal
  effects.
\newblock In \emph{UAI}, 2010.

\bibitem[Smola et~al.(1998)Smola, Frie{\ss}, and
  Sch{\"o}lkopf]{smola1998semiparametric}
A.~Smola, T.~T. Frie{\ss}, and B.~Sch{\"o}lkopf.
\newblock Semiparametric support vector and linear programming machines.
\newblock In \emph{NIPS}, 1998.

\bibitem[Tsiatis(2006)]{tsiatis2006semiparametric}
A.~A. Tsiatis.
\newblock \emph{Semiparametric Theory and Missing Data}.
\newblock Springer, 2006.

\bibitem[Van~der Laan and Rose(2011)]{van2011targeted}
M.~J. Van~der Laan and S.~Rose.
\newblock \emph{Targeted Learning: Causal Inference for Observational and
  Experimental Data}.
\newblock Springer, 2011.

\bibitem[Vershynin(2018)]{vershynin2018high}
R.~Vershynin.
\newblock \emph{High-Dimensional Probability: An Introduction with Applications
  in Data Science}.
\newblock Cambridge University Press, 2018.

\bibitem[Wainwright(2019)]{wainwright2019high}
M.~J. Wainwright.
\newblock \emph{High-Dimensional Statistics: A Non-Asymptotic Viewpoint}.
\newblock Cambridge University Press, 2019.

\bibitem[Wakefield(2013)]{wakefield2013bayesian}
J.~Wakefield.
\newblock \emph{Bayesian and Frequentist Regression Methods}.
\newblock Springer, 2013.

\bibitem[Wellner and Zhang(2007)]{wellner2007two}
J.~A. Wellner and Y.~Zhang.
\newblock Two likelihood-based semiparametric estimation methods for panel
  count data with covariates.
\newblock \emph{The Annals of Statistics}, 35\penalty0 (5), 2007.

\bibitem[Xia and H{\"a}rdle(2006)]{xia2006semi}
Y.~Xia and W.~H{\"a}rdle.
\newblock Semi-parametric estimation of partially linear single-index models.
\newblock \emph{Journal of Multivariate Analysis}, 97\penalty0 (5), 2006.

\end{thebibliography}
